\documentclass[twoside]{article}

\usepackage[accepted]{aistats2024}

\usepackage{microtype}
\usepackage{graphicx}
\usepackage{subfigure}
\usepackage{booktabs}
\usepackage[hidelinks]{hyperref}
\usepackage{algorithm}
\usepackage{algorithmic}
\usepackage{amsmath}
\usepackage{amssymb}
\usepackage{mathtools}
\usepackage{amsthm}
\usepackage{xcolor}
\usepackage{nicefrac}
\usepackage{textcase}
\usepackage{standalone}
\usepackage{tikz}
\usepackage{enumitem}
\usepackage{silence}
\usepackage{thm-restate}
\usepackage{thmtools}
\usepackage{xspace}
\usepackage[capitalize,noabbrev]{cleveref}

\usepackage[round]{natbib}

\theoremstyle{plain}
\newtheorem{theorem}{Theorem}
\newtheorem{proposition}[theorem]{Proposition}
\newtheorem{lemma}[theorem]{Lemma}
\newtheorem{corollary}[theorem]{Corollary}
\newtheorem{definition}[theorem]{Definition}

\newcommand{\onebf}{\ensuremath{\mathbf{1}}}
\newcommand{\Ibf}{\ensuremath{\mathbf{I}}}
\newcommand{\vbf}{\mathbf{v}}
\newcommand{\wbf}{\mathbf{w}}
\newcommand{\xbf}{\ensuremath{\mathbf{x}}}
\newcommand{\zerobf}{\ensuremath{{\mathbf 0}}}

\newcommand{\Dcal}{\ensuremath{\mathcal{D}}}
\newcommand{\Hcal}{\ensuremath{\mathcal{H}}}
\newcommand{\Lcal}{\ensuremath{\mathcal{L}}}
\newcommand{\Mcal}{\ensuremath{\mathcal{M}}}
\newcommand{\Ncal}{\ensuremath{\mathcal{N}}}
\newcommand{\Rcal}{\ensuremath{\mathcal{R}}}
\newcommand{\Scal}{\ensuremath{\mathcal{S}}}
\newcommand{\Tcal}{\ensuremath{\mathcal{T}}}
\newcommand{\Vcal}{\ensuremath{\mathcal{V}}}
\newcommand{\Xcal}{\ensuremath{\mathcal{X}}}
\newcommand{\Ycal}{\ensuremath{\mathcal{Y}}}

\newcommand{\Ebb}{\ensuremath{\mathbb{E}}}
\newcommand{\Nbb}{\ensuremath{\mathbb{N}}}
\newcommand{\Pbb}{\ensuremath{\mathbb{P}}}
\newcommand{\Rbb}{\ensuremath{\mathbb{R}}}

\newcommand{\R}{\Rbb}
\newcommand{\Rpe}{\R_{+}^{*}}

\newcommand{\LB}{\left[}
\newcommand{\RB}{\right]}
\newcommand{\LC}{\left\{}
\newcommand{\RC}{\right\}}
\newcommand{\RN}{\right\vert}
\newcommand{\LN}{\left\vert}
\newcommand{\LP}{\left(}
\newcommand{\RP}{\right)}

\newcommand{\ie}{{\it i.e.}\xspace}
\newcommand{\eg}{{\it e.g.}\xspace}
\newcommand{\resp}{{\it resp.}\xspace}
\newcommand{\wrt}{{\it w.r.t.}\xspace}
\newcommand{\iid}{{\it i.i.d.}\xspace}

\newcommand{\defeq}{\triangleq}

\newcommand{\indic}{\operatorname{I}}
\DeclareMathOperator*{\EE}{\Ebb}
\DeclareMathOperator*{\PP}{\Pbb}
\DeclareMathOperator*{\esssup}{\text{\rm esssup}}

\newcommand{\KL}{\operatorname{KL}}
\newcommand{\kl}{\operatorname{kl}}
\newcommand{\klmax}{\overline{\kl}}

\newcommand{\loss}{\ell}
\newcommand{\Risk}{\text{R}}
\newcommand{\RiskLoss}{\Risk^{\loss}}
\newcommand{\RiskLossp}{\Risk^{\loss'}}

\renewcommand{\P}{\pi}
\newcommand{\Q}{\rho}
\newcommand{\AQ}{\Q_{\Scal}}

\newcommand{\comp}{\operatorname{\mu}}
\newcommand{\uc}{{\tt u}}
\renewcommand{\aa}{{\tt a}}
\newcommand{\PhiUC}{\Phi_{\uc}}
\newcommand{\PhiA}{\Phi_{\aa}}
\newcommand{\PhiComp}{\Phi_{\comp}}

\newcommand{\vc}{{\tt vc}}
\newcommand{\rad}{{\tt rad}}

\newcommand{\distdistfro}{\distfro^{\boldsymbol{D}}}
\newcommand{\distdistltwo}{\distltwo^{\boldsymbol{D}}}
\newcommand{\distpathnorm}{\pathnorm^{\boldsymbol{D}}}
\newcommand{\distparamnorm}{\paramnorm^{\boldsymbol{D}}}
\newcommand{\distsumfro}{\sumfro^{\boldsymbol{D}}}
\newcommand{\distgap}{\gap^{\boldsymbol{D}}}
\newcommand{\distneural}{\neural^{\boldsymbol{D}}}

\newcommand{\riskdistfro}{\distfro^{\boldsymbol{R}}}
\newcommand{\riskdistltwo}{\distltwo^{\boldsymbol{R}}}
\newcommand{\riskpathnorm}{\pathnorm^{\boldsymbol{R}}}
\newcommand{\riskparamnorm}{\paramnorm^{\boldsymbol{R}}}
\newcommand{\risksumfro}{\sumfro^{\boldsymbol{R}}}
\newcommand{\riskgap}{\gap^{\boldsymbol{R}}}

\newcommand{\distfro}{\text{\sc DistFro}}
\newcommand{\distltwo}{\text{\sc DistL$_2$}}
\newcommand{\pathnorm}{\text{\sc PathNorm}}
\newcommand{\paramnorm}{\text{\sc ParNorm}}
\newcommand{\sumfro}{\text{\sc SumFro}}
\newcommand{\gap}{\text{\sc Gap}}
\newcommand{\neural}{\text{\sc Neural}}

\newif\ifnotappendix
\notappendixtrue

\begin{document}

\runningtitle{PAC-Bayes Generalization Bounds with Complexity Measures}

\runningauthor{Paul Viallard, Rémi Emonet, Amaury Habrard, Emilie Morvant, Valentina Zantedeschi}

\twocolumn[

\aistatstitle{Leveraging PAC-Bayes Theory and Gibbs Distributions for Generalization Bounds with Complexity Measures}

\aistatsauthor{Paul Viallard$^{*,}$\footnotemark \And Rémi Emonet$^{\dag,\diamond,\S}$ \And Amaury Habrard$^{\dag,\diamond,\S}$ \AND Emilie Morvant$^\dag$ \And Valentina Zantedeschi$^{\ddag}$}
\vspace{10pt}

\aistatsaddress{
$^*$ Univ Rennes, Inria, CNRS IRISA - UMR 6074, F35000 Rennes, France\\
$^\dag$ Université Jean Monnet Saint-Etienne, CNRS, Institut d Optique Graduate School,\\ Inria$^{\diamond}$, Laboratoire Hubert Curien UMR 5516, F-42023, SAINT-ETIENNE, FRANCE\\
$\S$ Institut Universitaire de France (IUF)\\
$\ddag$ ServiceNow Research\\
}]

\footnotetext{This research began when the author was affiliated with Laboratoire Hubert Curien and finished at Inria Paris.}

\begin{abstract}
In statistical learning theory, a generalization bound usually involves a complexity measure imposed by the considered theoretical framework.
This limits the scope of such bounds, as other forms of capacity measures or regularizations are used in algorithms.
In this paper, we leverage the framework of disintegrated PAC-Bayes bounds to derive a \emph{general} generalization bound instantiable with arbitrary complexity measures.
One trick to prove such a result involves considering a commonly used family of distributions: the Gibbs distributions.
Our bound stands in probability jointly over the hypothesis and the learning sample, which allows the complexity to be adapted to the generalization gap as it can be customized to fit both the hypothesis class and the task.
\end{abstract}

\section{INTRODUCTION}

Statistical learning theory offers various theoretical frameworks to assess generalization by studying whether the empirical risk is representative of the true risk.
This is often done by bounding a deviation, called the generalization gap, between these risks.
An upper bound on this gap is usually a function of two main quantities: {\it (i)} the size of the training set, {\it (ii)} a complexity measure that captures how prone a model is to overfitting.
The higher the complexity, the higher the number of examples needed to obtain a tight bound on the gap.
One limitation is that existing frameworks are restricted to specific complexity measures, \eg, the VC-dimension~\citep{vapnik1971uniform} or the Rademacher complexity~\citep{bartlett2002rademacher} (known to be large~\citep{nagarajan2019uniform}). 

Recently, \citet[Proposition 1]{lee2020neural} related arbitrary complexity measures to their usage in generalization bounds.
Indeed, if we interpret this bound, it says that the generalization gap is upper-bounded by a user-defined complexity measure with high probability if the complexity measure is close to the generalization gap.
However, this bound is uncomputable since it relies on a measure of closeness between the measure and the gap.
Hence, to our knowledge, there is no computable generalization bound able to capture, by construction, an arbitrary complexity measure that can serve as a good proxy for the generalization gap.

In this paper, we tackle this drawback by leveraging the framework of disintegrated PAC-Bayesian bounds (\Cref{theorem:general-disintegrated-rivasplata}) to propose a novel general generalization bound instantiable with arbitrary complexity measures.
To do so, we incorporate a user-defined parametric function characterizing the complexity in a probability distribution over the hypothesis set, expressed as a Gibbs distribution (also called Boltzmann distribution). 
This trick allows us to derive guarantees in terms of probabilistic bounds that depend on a model sampled from this user-parametrized Gibbs distribution.
It is worth noticing that our result is general enough to obtain bounds on well-known complexity measures such as the VC dimension or the Rademacher complexity.
We believe that our result provides new theoretical foundations for understanding the generalization abilities of machine learning models and for performing model selection in practice.
As an illustration, we empirically show how some arbitrary complexity measures, studied by~\citet{jiang2020fantastic,dziugaite2020search,jiang2021methods}, can be integrated into our framework. 
Moreover, inspired by \citet{lee2020neural}, we investigate how our bounds behave when provided with a complexity measure learned via a neural network.

\textbf{Paper's Organization.}
\Cref{sec:setting} provides preliminary definitions and concepts. 
Then, \Cref{sec:contrib} presents our framework.
In \Cref{sec:experiments}, we provide a practical instantiation of our framework.

\section{PRELIMINARIES}
\label{sec:setting}

\subsection{Notations and Setting}
\label{sec:notations}

We stand in a supervised classification setting where $\Xcal$ is the input space and $\Ycal$ is the label space.
An example $(\xbf, y)\!\in\! \Xcal{\times}\Ycal$ is drawn from an unknown data distribution~$\Dcal$ on $\Xcal{\times}\Ycal$.
A learning sample $\Scal{=}\{(\xbf_i, y_i)\}_{i=1}^m$ contains $m$ examples drawn \iid from $\Dcal$; we denote the distribution of such a sample by $\Dcal^m$.
Let $\Hcal$ be a possibly infinite set of hypotheses $h\!:\!\Xcal{\to}\Ycal$ that return a label from $\Ycal$ given an input from $\Xcal$. 
Let $\Mcal(\Hcal)$ be the set of strictly positive probability densities on $\Hcal$ given a reference measure (\eg, the Lebesgue measure).
Given $\Scal$ and a loss function $\loss\!:\! \Hcal{\times}(\Xcal{\times}\Ycal)\to\R$, we aim to find $h \!\in\! \Hcal$ that minimizes the true risk $\RiskLoss_{\Dcal}(h) {=} \EE_{(\xbf,y)\sim\Dcal} \loss(h, (\xbf,y))$.
As $\Dcal$ is unknown, $\RiskLoss_{\Dcal} (h)$ is in practice estimated with its empirical counterpart: the empirical risk $\RiskLoss_{\Scal}(h) {=} \frac{1}{m}\sum_{i=1}^{m}\loss(h, (\xbf_i, y_i))$.
We denote the generalization gap by $\phi\!:\! \R^2{\to}\R$, which quantifies how much the empirical risk is representative of the true risk; it is usually defined by $\phi(\RiskLoss_{\Dcal}(h), \RiskLoss_{\Scal}(h))\!=\!\vert\RiskLoss_{\Dcal}(h){-}\RiskLoss_{\Scal}(h)\vert$.

In this paper, we leverage the PAC-Bayesian setting~\citep{shawetaylor1997pac,mcallester1998some} to bound the generalization gap with a function involving an arbitrary measure of complexity (see \citet{guedj2019primer,hellstrom2023generalization,alquier2024user} for recent surveys).
In PAC-Bayes, we assume an \textit{apriori} belief on the hypotheses in $\Hcal$ modeled by a prior distribution $\P\!\in\!\Mcal(\Hcal)$ on $\Hcal$.
Instead of looking for the best $h\!\in\!\Hcal$, we aim to learn, from $\Scal$ and $\P$, a \textit{posterior} distribution $\Q\!\in\!\Mcal(\Hcal)$ \mbox{on $\Hcal$} to assign higher probability to the best hypotheses \mbox{in $\Hcal$} (the support of $\Q$ is included in the one of $\P$). 
A PAC-Bayesian generalization bound provides an upper bound in expectation \mbox{over $\Q$}, meaning it bounds the generalization gap expressed as $\vert\EE_{h\sim\Q}[\RiskLoss_{\Dcal}(h){-}\RiskLoss_{\Scal}(h)]\vert$. 
The complexity depends here on the KL divergence between $\Q$ and $\P$ defined as $\KL(\Q\|\P)=\EE_{h\sim\Q}\!\ln\frac{\Q(h)}{\P(h)}$.
This complexity captures how much $\Q$ and $\P$ deviate in expectation over all the hypotheses.
To incorporate a custom complexity in a bound, we follow a slightly different framework called the disintegrated PAC-Bayesian bound (see below) in which the expectations on $\Q$ are \textit{disintegrated}: for a single $h$ sampled from $\Q$, it upper-bounds the gap $\phi(\RiskLoss_{\Dcal}(h), \RiskLoss_{\Scal}(h)){=}\vert\RiskLoss_{\Dcal}(h){-}\RiskLoss_{\Scal}(h)\vert$.

\subsection{Disintegrated PAC-Bayesian Bounds}
\label{sec:disintegrated}

We recall now the framework of disintegrated PAC-Bayesian bounds (introduced by \citet[Th~1.2.7]{catoni2007pac} and \citet[Prop~3.1]{blanchard2007occams}) on which our contribution is based.
As far as we know, despite their significance, they have received little attention in the literature and have only received renewed interest for deriving tight bounds in practice recently (\eg, \citet{rivasplata2020pac,hellstrom2020generalization,viallard2024general}).
Such bounds provide guarantees for a hypothesis $h$ sampled from a posterior distribution $\AQ$, where $\AQ$ depends on the learning sample $\Scal\!\sim\! \Dcal^m$. 
In fact, these bounds stand with high probability (at least $1{-}\delta$) over the random choice of learning sample $\Scal{\sim} \Dcal^m$ \emph{and} a hypothesis $h$.
This paper mainly focuses on the bound of \citet[][Th.1~{\it (i)}]{rivasplata2020pac} recalled below in \Cref{theorem:general-disintegrated-rivasplata}.

\begin{restatable}[General Disintegrated Bound of~\citet{rivasplata2020pac}]{theorem}{generaldisintegratedrivasplata}\label{theorem:general-disintegrated-rivasplata}
For any distribution $\Dcal$ on $\Xcal{\times}\Ycal$, for any hypothesis set $\Hcal$, for any distribution $\P\!\in\!\Mcal(\Hcal)$, for any measurable function $\varphi: \Hcal\times(\Xcal{\times}\Ycal)^m\to \R$, for any $\delta\!\in\!(0, 1]$, we have with probability at least $1-\delta$ over $\Scal\sim\Dcal^m$ and $h\sim\AQ$
\begin{align*}
\varphi(h,\Scal) \le \ln\frac{\AQ(h)}{\P(h)}\!+\!\ln\!\left[\frac{1}{\delta}\EE_{\Vcal\sim\Dcal^m}\EE_{g\sim\P}e^{\varphi(g,\Vcal)}\right],
\end{align*}
where $\AQ\in\Mcal(\Hcal)$ is a posterior distribution.
\end{restatable}

Remark that $\varphi$ can be any (measurable) function. 
However, it is usually defined as $\varphi(h,\Scal)=m\,\phi(\RiskLoss_{\Dcal}(h), \RiskLoss_{\Scal}(h))$, which is a deviation between the true risk $\RiskLoss_{\Dcal}(h)$ and the empirical risk $\RiskLoss_{\Scal}(h)$.
The bound depends on two terms: \mbox{{\it (a)} the} {\it disintegrated} KL divergence $\ln\tfrac{\AQ(h)}{\P(h)}$ defining how much $\P$ and $\AQ$ deviate for a single $h$, {\it (b)} the term $\ln\!\LB\frac{1}{\delta}\EE_{\Vcal}\EE_{g} \exp\LP\varphi(g,\Vcal)\RP\RB$ which is constant \wrt $h\!\in\!\Hcal$ and $\Scal\!\in\!(\Xcal{\times}\Ycal)^m$. 
Note that, to instantiate the bound with a given $\varphi$, the right-most term {\it (b)} is usually upper-bounded.
In fact, it is constant \wrt the hypothesis $g{\sim}\P$ and the learning sample $\Vcal{\sim}\Dcal^m$.
Then, to integrate the relevance of the prior belief and for the sake of simplicity, in the rest of the paper, we refer to as ``\textit{complexity measure}'' the right-hand side of the bound.
This is in slight contrast with the standard definition of complexity (\eg, in the case of the VC-dimension or the Rademacher complexity), where the term {\it (b)} is not included in the definition.

\looseness=-1
In the bound of \Cref{theorem:general-disintegrated-rivasplata}, the disintegrated KL divergence suffers from drawbacks: the KL complexity term is imposed by the framework and can be subject to high variance in practice~\citep{viallard2024general}.
Despite this shortcoming, it is important to notice that the disintegrated KL divergence has a clear advantage: it only depends on the sampled hypothesis $h\sim\AQ$ and the data sample $\Scal$, instead of the whole hypothesis class (as it is the case, for instance, with the Rényi divergence in the disintegrated PAC-Bayesian bounds of \citet{viallard2024general}, or with the bounds based on the VC-dimension, or the Rademacher complexity). 
This might imply a better correlation between the generalization gap and some complexity measures.
In the next section, we leverage the disintegrated KL divergence to derive our main contribution: a general bound that involves arbitrary complexity measures. 

\section{INTEGRATING MEASURES IN GENERALIZATION BOUNDS}
\label{sec:contrib} 

In \cref{sec:def,sec:gibbs-optim}, we give intuitions about our contribution and recall notions about the Gibbs distribution.
Second, we formalize our result in \Cref{sec:result}.

\subsection{The Framework}
\label{sec:def}
The idea to introduce our notion of complexity measure is to parametrize the complexity with an additional ``customizable'' function $\comp\!:\! \Hcal{\times}(\Xcal{\times}\Ycal)^m {\to} \R$ that we call \textit{parametric function}. 
Thanks to the function $\comp$, we define the randomized complexity measure $\PhiComp^{r}(h{,}\Scal{,}\delta)$ as a real-valued function parameterized by $\comp$ and an external randomness $r\!\sim\!\Rcal$ which takes as argument a hypothesis $h\!\in\!\Hcal$, a learning sample $\Scal{\in}(\Xcal{\times}\Ycal)^m$, \mbox{and $\delta$}.
As we will see in \Cref{sec:result}, the bound we derive in \Cref{theorem:disintegrated-comp} depends on the complexity measure $\PhiComp^{r}(h, \Scal, \delta)$ and takes the following form.

\begin{figure}
\includegraphics[width=1.0\linewidth]{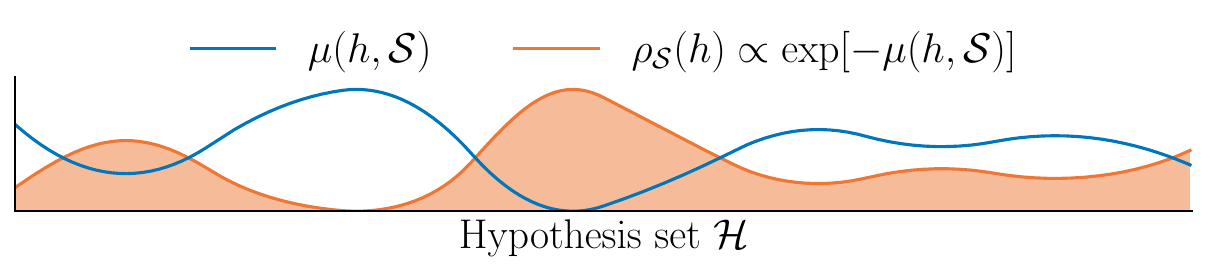}
\caption{\label{fig:gibbs}
\looseness=-1
 Illustration of the behavior of the Gibbs distribution $\AQ$ with a parametric function $\comp$. 
 The x-axis represents a (continuous) hypothesis set, and the y-axis the values of $\AQ$ and $\comp$.
 The distribution $\AQ$ gives a higher probability to the hypotheses with a low $\comp$ value.
}
\end{figure}

\begin{definition}
\label{def:comp-bound}\looseness=-1
Let $\loss\!:\! \Hcal{\times}(\Xcal{\times}\Ycal){\to}\R$ be a loss function, $\phi\!:\! \R^2{\to}\R$ be the generalization gap, and $\comp\!:\! \Hcal{\times}(\Xcal{\times}\Ycal)^m {\to} \R$ be a parametric function.
A generalization bound with a complexity measure is defined such that if for any distribution $\Dcal$ on $\Xcal{\times}\Ycal$, for any distribution $\Rcal$ representing the randomness, for any hypothesis set $\Hcal$, there exists a randomized real-valued function $\PhiComp^{r}: \Hcal{\times}(\Xcal{\times}\Ycal)^m{\times} (0,1] {\to} \R$ such that for any  $\delta\!\in\!(0, 1]$, we have
\begin{align*}
    \PP_{r\sim\Rcal,\Scal\sim\Dcal^m,h\sim\AQ}\!\Big[ \phi(\RiskLoss_{\Dcal}(h){,}\RiskLoss_{\Scal}(h)) \le \PhiComp^{r}(h, \Scal, \delta) \Big]{\ge}1{-}\delta,
\end{align*}
where $\AQ\in\Mcal(\Hcal)$ is a posterior distribution. 
\end{definition}
\looseness=-1
The main trick to obtain a bound that involves a parametrizable complexity measure is to consider a posterior distribution $\AQ$ that depends on $\comp$.
To do so, we propose to set $\AQ$ as the Gibbs distribution defined as
\begin{align}
\label{eq:gibbs-distribution}
   \AQ(h) \propto \exp\LB-\comp(h, \Scal)\RB.
\end{align}
This formulation might look restrictive, but it can represent any probability density function provided that a relevant complexity measure is selected. 
For instance, let $\AQ'$ be a distribution on $\Hcal$, \eg, a Gaussian or a Laplace distribution, by setting $\comp(h, \Scal) \!=\! - \ln\AQ'(h)$ we can retrieve the distribution $\AQ'$.
Moreover, this Gibbs distribution $\AQ$ is interesting from an optimization viewpoint: given a fixed learning \mbox{sample $\Scal$}, a hypothesis $h$ is more likely to be sampled from it when $\comp(h, \Scal)$ is low (see \Cref{fig:gibbs} for an illustration).
In fact, the function $h\mapsto\comp(h, \Scal)$ can be seen as an objective function.
For instance, to minimize the true risk $\RiskLoss_{\Dcal}(h)$, one can ideally set $\comp(h,\Scal) \!=\! \alpha\RiskLoss_{\Dcal}(h)$ that is associated with a Gibbs distribution which samples hypotheses with small true risks and concentrates around the small risks when $\alpha\!\in\!\Rpe$ increases.
However, since the true risk is unknown, it must be replaced with a computable function $\comp$.
For instance, $\comp$ can be the empirical risk such as $\comp(h,\Scal) \!=\! \alpha\RiskLoss_{\Scal}(h)$.  

\subsection{Gibbs Distribution and Optimization}
\label{sec:gibbs-optim}

Given a differentiable parametric function defined by $\comp(h,\Scal)\!=\!\alpha\nu(h,\Scal)$ (with $\alpha$ a concentration parameter), its associated Gibbs distribution can be related to the Stochastic Gradient Langevin Dynamics algorithm~\citep[SGLD,][]{welling2011bayesian} that learns the hypothesis $h\!\in\!\Hcal$ by running iterations of the form 
\begin{align}
    h_t \longleftarrow h_{t-1} - \eta\nabla\nu(h_t,\Scal) + \sqrt{\frac{2\eta}{\alpha}}\epsilon_t, \label{eq:sgld}
\end{align}
where $\epsilon_t\!\sim\!\Ncal(\zerobf, \Ibf_D)$, and $h_t$ is the hypothesis learned at iteration $t\!\in\!\Nbb$, and $\eta$ is the learning rate, and $\alpha$ is the concentration parameter for the Gibbs distribution.
When $\alpha$ increases, the noise $\epsilon_t$ has less influence on the next iterate obtained from SGLD as $\sqrt{\nicefrac{2\eta}{\alpha}}\epsilon_t$ decreases, and hence, minimizes better the function $\nu$.
Moreover, when the learning rate $\eta$ tends to zero, SGLD  becomes a continuous-time process called Langevin diffusion, defined as the stochastic differential equation in \Cref{eq:sgld-2}.
Indeed, \Cref{eq:sgld} can be seen as the Euler-Maruyama discretization \citep[see,][]{raginsky2017nonconvex} of \Cref{eq:sgld-2} defined for $t\ge 0$ as
\begin{align}
    dh_t = -\nabla \nu(h_t, \Scal)dt + \sqrt{\frac{2}{\alpha}}B_t,\label{eq:sgld-2}
\end{align}
where $B_t$ is the Brownian motion.
Under some mild assumptions on the function $\nu$, \citet{chiang1987diffusion} show that the invariant distribution of the Langevin diffusion is the Gibbs distribution $\AQ$ with $\comp(h,\Scal)=\alpha\nu(h,\Scal)$.

\subsection{Bounds with Complexity Measures}
\label{sec:result}

We now introduce our main results, \ie, generalization bounds with user-defined complexity measures. 
In \Cref{sec:result-general}, we present a general theorem that fulfills \Cref{def:comp-bound}.
We specialize our result to uniform priors in \Cref{sec:result-practical-uniform} and informed priors in \cref{sec:result-practical-informed}.

\subsubsection{General Generalization Bound}
\label{sec:result-general}

We state below our theorem that introduces a bound on the generalization gap involving the parametric function $\comp$, which stands for hypotheses sampled from the posterior distribution $\AQ(h) \!\propto\! \exp\LB-\!\comp(h, \Scal)\RB$.
Note that our bound is ``general,'' meaning the generalization gap $\phi$ has to be further upper-bounded.

\begin{restatable}{theorem}{theoremboundcomplexitymeasure}\label{theorem:disintegrated-comp}
Let $\loss: \Hcal{\times}(\Xcal{\times}\Ycal){\to}\R$ be a loss function and $\phi\!:\! \R^2{\to}\R$ be a generalization gap.
For any $\Dcal$ on $\Xcal{\times}\Ycal$, for any hypothesis set $\Hcal$, for any prior distribution $\P\in\Mcal(\Hcal)$ on $\Hcal$, for any $\comp\!:\! \Hcal{\times}(\Xcal{\times}\Ycal)^m{\to}\R$, for any $\delta\!\in\!(0, 1]$, we have with probability at least $1-\delta$ over $h'\sim\P$, $\Scal\sim\Dcal^m$, and $h\sim\AQ$
\ifnotappendix%
\begin{align*}
    \phi(\RiskLoss_{\Dcal}(h), \RiskLoss_{\Scal}(h)) &\le \comp(h'\!,\Scal) - \comp(h,\Scal) +\ln\frac{\P(h')}{\P(h)}\\[-1.5mm]
    &+ \ln\!\left[\frac{4}{\delta^2} \EE_{\Vcal\!{\sim}\Dcal^m}\EE_{g{\sim}\P} e^{\phi(\RiskLoss_{\Dcal}(g),\RiskLoss_{\Vcal}(g))}\right]\\
    &\defeq \PhiComp^{h'}(h,\Scal,\delta),
\end{align*}
\else%
\begin{align*}
    \phi(\RiskLoss_{\Dcal}(h), \RiskLoss_{\Scal}(h)) &\le \comp(h'\!,\Scal) - \comp(h,\Scal) +\ln\frac{\P(h')}{\P(h)} + \ln\!\left[\frac{4}{\delta^2} \EE_{\Vcal\!{\sim}\Dcal^m}\EE_{g{\sim}\P} e^{\phi(\RiskLoss_{\Dcal}(g),\RiskLoss_{\Vcal}(g))}\right] \defeq \PhiComp^{h'}(h,\Scal,\delta),
\end{align*}
\fi
where $\AQ$ is the Gibbs distribution as in \Cref{eq:gibbs-distribution}.
\end{restatable}

The bound $\PhiComp^{h'}(h{,}\Scal{,}\delta)$ of \Cref{theorem:disintegrated-comp} depends on three terms: {\it (i)} the difference $\comp(h'{,}\Scal){-}\comp(h{,}\Scal)$, \mbox{{\it (ii)} the} log ratio $\ln({\P(h')}/{\P(h)})$, {\it (iii)} a constant term $\ln[\frac{4}{\delta^2}\EE_{\Vcal\!{\sim}\Dcal^m}\EE_{g{\sim}\P} \exp[\phi(\RiskLoss_{\Dcal}(g),\RiskLoss_{\Vcal}(g))]]$.
Compared to \Cref{theorem:general-disintegrated-rivasplata}, we essentially upper-bound the disintegrated KL divergence $\ln\frac{\AQ(h)}{\P(h)}$ by the difference $\comp(h'{,}\Scal){-}\comp(h{,}\Scal)$ and the log ratio $\ln({\P(h')}/{\P(h)})$. 
The advantage of these two terms is that they are easily computable, as long as we can compute $\comp(h',\Scal)$, $\comp(h,\Scal)$ and the density of $\P$ (up to its normalization constant).
This is in contrast with the the result of \citet{lee2020neural}, that is essentially a bound that holds for all $\epsilon>0$ and of the form,
\begin{align*}
\PP_{\Scal\sim\Dcal^m,h\sim\AQ}\Big[ |\Risk_{\Dcal}(h){-}\Risk_{\Scal}(h)| \le \comp(h, \Scal){+}\epsilon \Big]\ge 1{-}\delta'(\epsilon),
\end{align*}
where $\delta'(\epsilon)$ depends on $n$ learning samples $\Scal_1,\dots,\Scal_n$, on $n$ hypotheses $h_1\sim \Q_{\Scal_1},\dots,h_n\sim\Q_{\Scal_n}$, and on the unknown distribution $\Dcal$.
This result has no restriction on the form of the distribution $\AQ$, however, the dependence on $\Dcal$ makes the term $\delta'(\epsilon)$ not computable (in contrast to our bound); see \Cref{sec:comparison-lee} for more details.
Note also that the term {\it (iii)} is usually negligible compared to {\it (i)} and {\it (ii)}, and it is upper-bounded when the generalization gap $\phi$ is instantiated.
To get a bound that converges when $m$ increases, it is sufficient to set $\phi$ as a function of $m$ as it is done further.
The tightness of the term {\it (ii)} depends on the instantiation of $\P$; we propose two types of instantiation in \Cref{sec:result-practical-uniform,sec:result-practical-informed}. 
Lastly, the term {\it (i)} depends on the choice of $\comp$ which has a big influence on the sampled hypothesis $h\!\sim\!\AQ$ and so on the gap $\phi(\RiskLoss_{\Dcal}(h), \RiskLoss_{\Scal}(h))$. 
For instance, when $\comp(h,\Scal){=}0$, the difference $\comp(h', \Scal){-}\comp(h, \Scal){=}0$, but, in this case, the posterior distribution $\AQ$ is uniform which does not permit to sample a hypothesis minimizing the true risk $\RiskLoss_{\Dcal}(h)$.
There is hence a trade-off to find between minimizing this difference and sampling a hypothesis minimizing the gap $\phi(\RiskLoss_{\Dcal}(h), \RiskLoss_{\Scal}(h))$ and the true risk $\RiskLoss_{\Dcal}(h)$.
In \Cref{sec:experiments}, we see how to instantiate the parametric function $\mu$.
Note that, when instantiated correctly, it also allows to get uniform-convergence-based and algorithm-dependent bounds; see \Cref{sec:comparison-uc-algo}.

\subsubsection{Practical Bound with Uniform Priors}
\label{sec:result-practical-uniform}

The remaining challenge to get a practical bound is to upper-bound  $\ln[\frac{4}{\delta^2}\EE_{\Vcal\!{\sim}\Dcal^m}\EE_{g{\sim}\P}\exp[\phi(\RiskLoss_{\Dcal}(g), \RiskLoss_{\Vcal}(g))]]$ and $\ln({\P(h')}/{\P(h)})$.
As an illustration, we restrict ourselves in the rest of the paper to the case where the loss is bounded, \ie, we consider a loss function $\loss \!:\! \Hcal{\times}(\Xcal{\times}\Ycal) \!\to\! [0,1]$. 
Under this assumption, we provide in the next corollary an instantiation of \Cref{theorem:disintegrated-comp} for the generalization gap $\phi(\RiskLoss_{\Dcal}(h), \RiskLoss_{\Scal}(h)){=}m\kl[\RiskLoss_{\Scal}(h)\|\RiskLoss_{\Dcal}(h)]$ where $\kl(q\|p) \defeq q\ln\frac{q}{p}+(1{-}q)\ln\frac{1{-}q}{1{-}p}$ for $p\in(0,1)$ and $q\in[0,1]$ and with a uniform distribution $\P$ on a bounded set $\Hcal$.
\begin{restatable}{corollary}{corollarydisintegratedcompunif}\label{corollary:disintegrated-comp-unif}
\looseness=-1
For any $\Dcal$ on $\Xcal{\times}\Ycal$, for any bounded hypothesis set $\Hcal$, given the uniform prior $\P$ on $\Hcal$, for any loss $\loss : \Hcal\times(\Xcal{\times}\Ycal) \to [0,1]$, for any $\comp\!:\! \Hcal{\times}(\Xcal{\times}\Ycal)^m{\to}\R$, for any $\delta\!\in\!(0,1]$, with probability at least $1{-}\delta$ over $\Scal{\sim}\Dcal^m$, $h'{\sim}\P$, $h{\sim}\AQ$ we have
\ifnotappendix%
\begin{align}
\kl\!\LB\RiskLoss_{\Scal}(h)\|\RiskLoss_{\Dcal}(h)\RB\le\frac{\comp(h'\!,\Scal){-}\!\comp(h{,}\Scal){+}\ln\tfrac{8\sqrt{m}}{\delta^2}}{m},\label{eq:disintegrated-comp-seeger}
\end{align}
\else%
\begin{align*}
\kl\!\LB\RiskLoss_{\Scal}(h)\|\RiskLoss_{\Dcal}(h)\RB\le \frac{1}{m}\LB\comp(h'\!,\Scal){-}\!\comp(h{,}\Scal){+}\ln\tfrac{8\sqrt{m}}{\delta^2}\RB,
\end{align*}
\fi
\looseness=-1
with $\AQ$ defined in \Cref{eq:gibbs-distribution}.
\end{restatable}

Interestingly, \Cref{corollary:disintegrated-comp-unif} gives a computable bound on $\kl[\RiskLoss_{\Scal}(h)\|\RiskLoss_{\Dcal}(h)]$.
From \Cref{eq:disintegrated-comp-seeger}, we obtain the following generalization bounds on the true risk $\RiskLoss_{\Dcal}(h)$:
\begin{align}
\RiskLoss_{\Dcal}(h)\le\klmax\LB\RiskLoss_{\Scal}(h) \;\middle|\; \frac{\comp(h'\!,\Scal){-}\comp(h{,}\Scal){+}\ln\tfrac{8\sqrt{m}}{\delta^2}}{m}\,\RB\!,\label{eq:disintegrated-comp-seeger-risk}
\end{align}
with $\klmax[q | \tau]{=}\max\{ p \!\in\! (0,1) \;|\; \kl(q\|p) {\le} \tau\}$.  
We use these bounds in \Cref{sec:experiments} to illustrate the generalization guarantees for different parametric \mbox{functions $\comp$}.
For some trivial cases, the convergence rate can be arbitrary, \eg, when $\comp(h\!, \Scal)\!=\!m\RiskLoss_{\Scal}(h)$.
For example, for a large empirical risk $\RiskLoss_{\Scal}(h')$ (which is common when $h'$ is sampled from a uniform prior on $\Hcal$), the right-hand side of \Cref{eq:disintegrated-comp-seeger} simplifies to $\PhiComp^{h'}(h{,}\Scal{,}\delta)\!=\![(\RiskLoss_{\Scal}(h'){-}\RiskLoss_{\Scal}(h)){+}\frac{1}{m}\ln({2\sqrt{m}}/{\delta})]_+$ and is large, for all $m$.
In order for the bound to be meaningful, we have to set $\comp$ such that the distribution $\AQ$ allows to sample $h$ minimizing the empirical risk $\RiskLoss_{\Scal}(h)$ and the generalization gap, and we want the complexity measure $\PhiComp^{h'}(h, \Scal, \delta)$ to be tight (with $h'\!\sim\!\P$).

\subsubsection{Practical Bound with Informed Priors}
\label{sec:result-practical-informed}

\looseness=-1
While it is common to consider uninformed priors when we have no \textit{apriori} belief, informative priors can be necessary and useful to get better results.
For that purpose, one solution is to consider distribution-dependent priors, heavily used in PAC-Bayes~\citep[see \eg,][]{parradohernandez2012pac,dziugaite2021role,perezortiz2021tighter}.
We use a strategy similar to that for the posterior $\AQ$ by defining the prior $\P$ as follows
\begin{align}
\label{eq:gibbs-distribution-prior}
   \P(h) \propto \exp\LB-\omega(h)\RB,
\end{align}
\looseness=-1
where $\omega\!:\! \Hcal {\to} \R$ can depend on the distribution $\Dcal$.
Hence, the prior can depend on an additional learning sample $\Scal'\!\in\!(\Xcal{\times}\Ycal)^{m'}$ sampled from $\Dcal$.
We prove the following corollary with the prior of \Cref{eq:gibbs-distribution-prior}.

\begin{restatable}{corollary}{corollarydisintegratedcompdata}\label{corollary:disintegrated-comp-data}
For any $\Dcal$ on $\Xcal{\times}\Ycal$, for any hypothesis set $\Hcal$, for any loss $\loss : \Hcal\times(\Xcal{\times}\Ycal) \to [0,1]$, for any $\comp\!:\! \Hcal{\times}(\Xcal{\times}\Ycal)^m{\to}\R$, for any $\omega\!:\! \Hcal{\to}\R$, for any $\delta\!\in\!(0,1]$, with probability at least $1{-}\delta$ over $\Scal{\sim}\Dcal^m$, $h'{\sim}\P$, $h{\sim}\AQ$ we have
\allowdisplaybreaks
\ifnotappendix%
\begin{align}
    \kl\!\LB\RiskLoss_{\Scal}(h)\|\RiskLoss_{\Dcal}(h)\RB\le
    &\frac{1}{m}\Big[[\comp(h'\!,\Scal){-}\omega(h')]\nonumber\\
    &-[\comp(h{,}\Scal){-}\omega(h)]+\ln\tfrac{8\sqrt{m}}{\delta^2} \Big],\label{eq:disintegrated-comp-seeger-data}
\end{align}
\else%
\begin{align}
    \kl\!\LB\RiskLoss_{\Scal}(h)\|\RiskLoss_{\Dcal}(h)\RB\le \frac{1}{m}\Big[[\comp(h'\!,\Scal){-}\omega(h')]-[\comp(h{,}\Scal){-}\omega(h)]+\ln\tfrac{8\sqrt{m}}{\delta^2} \Big],
\end{align}
\fi
with  $\AQ$ and $\P$ resp. defined in \Cref{eq:gibbs-distribution,eq:gibbs-distribution-prior}.
\end{restatable}

\section{USING COMPLEXITY MEASURES IN PRACTICE}
\label{sec:experiments}

\Cref{sec:experiments-setting} presents our experimental setting.
In \Cref{sec:experiments-emp-risk}, we first compare \Cref{corollary:disintegrated-comp-unif,corollary:disintegrated-comp-data} to two additional bounds with $\comp(h,\Scal)\!=\!\alpha\RiskLossp_{\Scal}(h)$ (where $\loss'$ is a differentiable loss).
In \Cref{sec:experiments-reg-risk} we study the behavior of our bounds when the parametric function $\mu$ is defined as a regularized empirical risk.
Finally, in \Cref{sec:experiments-neural-comp} we assess the tightness of our bounds when the complexity term is learned with a neural network.

\subsection[General Experimental Setting]{General Experimental Setting\protect\footnote{The source code is available at \url{https://github.com/paulviallard/AISTATS24-Complexity-Measures}.}}\label{sec:experiments-setting}

\looseness=-1
In this section, we investigate the tightness of \Cref{corollary:disintegrated-comp-unif,corollary:disintegrated-comp-data}'s bounds on the MNIST~\citep{lecun1998mnist} and FashionMNIST~\citep{xiao2017fashion} datasets. 
Specifically, we consider the bounds on the true risk and the empirical risk endowed with the 01-loss $\loss(h,(\xbf,y))\!=\!\indic[h(\xbf) \ne y]$ where $\indic[a]\!=\!1$ if $a$ is true and $0$ otherwise.

\textbf{Model.} 
Inspired by the setting of \citet{viallard2024general}, we train an All Convolutional Network~\citep{springenberg2015striving} that is fitted for the two datasets MNIST and FashionMNIST.
This network comprises $4$ convolutional layers of $10$ channels and a kernel of size $5{\times}5$ followed by a Leaky ReLU activation function (where the padding and the stride are set to $1$ except for the second layer where the stride is set to $2$). 
Finally, the network ends with an average pooling of size $8{\times}8$ followed by a Softmax activation function.
The weights are initialized with the Xavier Glorot uniform initializer~\citep{glorot2010understanding}, and the biases are initialized uniformly between $-1/\sqrt{250}$ and $+1/\sqrt{250}$ for all biases instead for the first layer they are initialized uniformly in $[-1/5, +1/5]$.

\textbf{Datasets.}
We keep the original test set $\Tcal$ to estimate the true risk that we refer to as test risk $\Risk_{\Tcal}(h)$.
To evaluate \Cref{corollary:disintegrated-comp-unif}'s bounds and to sample $h\!\sim\!\AQ$, 
we keep the original learning set $\Scal$ to evaluate.
To evaluate \Cref{corollary:disintegrated-comp-data}'s bound, and to sample $h\!\sim\!\AQ$ and $h'\!\sim\!\P$, 
the original learning set is split into $2$ sets $\Scal$ and $\Scal'$ respectively of size $m$ and $m'$; 
When $\frac{m'}{m+m'}{=}0$, the prior distribution is the uniform (non-data-dependent) and we retrieve \Cref{corollary:disintegrated-comp-unif}.

\looseness=-1
\textbf{Sampling and Bound Computation.}
To compute \Cref{corollary:disintegrated-comp-unif,corollary:disintegrated-comp-data}'s bounds, we aim to sample $h\!\sim\!\AQ$ and $h'\!\sim\!\P$ by performing SGLD\footnote{\citet{dziugaite2018data} also perform SGLD to sample from a Gibbs distribution.} (described in \Cref{eq:sgld}) and to evaluate these hypotheses with $\klmax$.
Note that using SGLD is efficient for sampling since it does not require computing the normalization constants of the two distributions.
To tune the learning rate, {\it (1)} we compute the mean loss over the learning sample (without training), {\it (2)} we start with a learning rate of size $0.1$, and we decrease it (by a factor of $0.1$) and reinitialize the model after each epoch if the mean loss is not decreasing (to retrain from scratch), {\it (3)} if the learning rate attains $10^{-10}$, we set the learning rate to its starting value $0.1$ and start learning from scratch.
Once the initial learning rate that decreases the mean loss is set, we perform SGLD for $10$ epochs and with a mini-batch of size $64$.
After each epoch, we decrease the learning rate by a factor of $0.5$.
Whenever we have to sample a risk value with SGLD, we replace the 01-loss by the differentiable bounded cross entropy of~\citet{dziugaite2018data} $\loss'(h, (\xbf,y)) {=} -\frac{1}{4}\ln(e^{-4}{+}(1{-}2e^{-4})h(\xbf)[y])$, where $h[y]$ is the probability assigned to the label $y$ by $h$.
The advantage of \citet{dziugaite2018data}'s cross-entropy is that it lies in $\loss(h, (\xbf,y)) \in [0,1]$, whereas the classical cross-entropy is unbounded.
For all experiments, we perform 5 runs to obtain a mean and a standard deviation, which involves sampling from $\AQ$ and $\P$ for each evaluation of the bounds and risks.

\subsection{Experiments on the Empirical Risk}
\label{sec:experiments-emp-risk}

In this section, we compare the tightness of our bounds of \Cref{corollary:disintegrated-comp-unif,corollary:disintegrated-comp-data} with bounds that share similarities with the literature.
More precisely, this comparison is done for the Gibbs distribution $\AQ$ defined with the parametric function $\mu(h,\Scal)\!=\!\alpha\RiskLossp_{\Scal}(h)$.
Indeed, this Gibbs distribution was already studied in the classical and disintegrated PAC-Bayesian theory but led to uncomputable bounds. 
We adapt these bounds to make them computable so that we can report them as baselines (more details are given in \Cref{sec:comparison-literature}).
More precisely, we compare our bounds to the following one (similar to \citet{lever2013tighter}): with probability at least $1{-}\delta$, we have
\newcommand{\eqcomparisonlever}{%
\ifnotappendix%
\begin{align}
\kl[\RiskLoss_{\Scal}(h)\|\RiskLoss_{\Dcal}(h)]\!\le\!{\textstyle\frac{1}{m}\!\!\LB\frac{\alpha^2}{8m}{+}\sqrt{\!\frac{\alpha^2}{2m}\!\ln\!\frac{6\sqrt{m}}{\delta}}{+}\ln\!\frac{6\sqrt{m}}{\delta}\RB}.\label{eq:comparaison-lever}
\end{align}
\else%
\begin{align*}
\kl[\RiskLoss_{\Scal}(h)\|\RiskLoss_{\Dcal}(h)]\le\frac{1}{m}\!\!\LB\frac{\alpha^2}{8m}{+}\sqrt{\frac{\alpha^2}{2m}\ln\frac{6\sqrt{m}}{\delta}}{+}\ln\frac{6\sqrt{m}}{\delta}\RB,
\end{align*}
\fi}
\eqcomparisonlever
We also adapt the proof technique of \citet{dziugaite2018data} to obtain with probability at least $1{-}\delta$
\newcommand{\eqcomparisondziugaite}{%
\ifnotappendix%
\begin{align}
\kl[\RiskLoss_{\Scal}(h) \|\RiskLoss_{\Dcal}(h)] \le \frac{1}{m}\Big[ \alpha\!\LB\RiskLossp_{\Scal}(h'){-}\RiskLossp_{\Scal}(h)\RB\nonumber\\
+ \alpha'\Big[\RiskLossp_{\Scal}(h){-}\RiskLossp_{\Scal}(h')\Big] + 2\alpha' + \ln\!\tfrac{8\sqrt{m}}{\delta^2}\Big],\label{eq:comparaison-dziugaite}
\end{align}
\else%
\begin{align*}
\kl[\RiskLoss_{\Scal}(h) \|\RiskLoss_{\Dcal}(h)] \le \frac{1}{m}\Big[ \alpha\!\LB\RiskLossp_{\Scal}(h'){-}\RiskLossp_{\Scal}(h)\RB + \alpha'\Big[\RiskLossp_{\Scal}(h){-}\RiskLossp_{\Scal}(h')\Big] + 2\alpha' + \ln\!\tfrac{8\sqrt{m}}{\delta^2}\Big],
\end{align*}
\fi}
\eqcomparisondziugaite
\looseness=-1
where $h'\!\sim\! \P$ and $\P(h)\!\propto\!\exp[-\alpha'\RiskLossp_{\Scal}(h)]$.
For \Cref{corollary:disintegrated-comp-data}, we set the prior $\P$ with the parametric function $\omega(h)\!=\!\alpha\RiskLossp_{\Scal'}(h)$ and with a learning sample $\Scal'$ satisfying the split ratio $\frac{m'}{m\!+\!m'} \!=\! 0.5$.
For all the bounds, we take $\alpha$ (and $\alpha'$) uniformly spaced on the logarithmic scale between $\sqrt{m}$ and $m$.
For each parameter $\alpha$ and bound, we select the prior minimizing the bound and average its value over $5$ runs.
We report in \Cref{fig:emp-risk} the evolution of the different bounds and the test risks \wrt to $\alpha$.

\textbf{Analysis.} As expected, the standard deviations of all the bounds are small only for large $\alpha$, as this parameter controls the concentration of the Gibbs distributions.
A larger $\alpha$ tends to imply lower test risks $\RiskLoss_{\Tcal}(h)$.
However, the bounds become large as $\alpha$ increases except for our bound of~\Cref{corollary:disintegrated-comp-data}.
This is an expected behavior in the case of \Cref{eq:comparaison-lever} since the bound increases when $\alpha$ increases.
For \Cref{corollary:disintegrated-comp-unif}, the bound is large when the difference $\alpha[\RiskLossp_{\Scal}(h')\!-\!\RiskLossp_{\Scal}(h)]$ is large.
This is effectively the case because $\RiskLossp_{\Scal}(h')$ is large since $h'$ is sampled from a uniform distribution and $\RiskLossp_{\Scal}(h)$ is small because $h$ is sampled from $\AQ(h)\propto \exp[-\alpha\RiskLossp_{\Scal}(h)]$.
The same phenomenon arises with \Cref{eq:comparaison-dziugaite} since $\RiskLossp_{\Scal}(h')$ is large when $\alpha'$ is small, \ie, the concentration is not sufficient to minimize the empirical risk.
The tightness of \Cref{corollary:disintegrated-comp-data} comes from the fact that both empirical risks for $h'\!\sim\!\P$ and $h\!\sim\!\AQ$ are small, and so is the bound when the risks $\RiskLossp_{\Scal}(h')$ and $\RiskLossp_{\Scal'}(h)$ are \mbox{small as well}.
Moreover, note that, for small $\alpha$, the test risks and the bound values are higher compared to the others.
This is due to the fact that we use half of the data ($\frac{m'}{m+m'}{=}0.5$) for learning an informed prior.
Indeed, the value of $\alpha$ is twice as small as for the other bounds, which makes the bound values and the test risks higher as the Gibbs distribution is less concentrated.

\subsection{Experiments on Regularized Risks}
\label{sec:experiments-reg-risk}

\looseness=-1
In order to tighten the bounds in \Cref{corollary:disintegrated-comp-unif,corollary:disintegrated-comp-data}, one might assume that selecting a hypothesis with a small trade-off between its empirical risk and a norm is a reasonable solution. 
Against all odds, we will see in this section that regularizing the empirical risk with a parametric function does not help to tighten the bounds.
To define the norms used as regularizers, we assume that the model $h$, composed of $L$ layers, is parameterized by weights (and biases) $\wbf\!\in\!\R^d$; we denote by $h_{\wbf^2}$ the hypothesis $h$ that has its weights replaced by $\wbf^2$.
We define by $\wbf_i$ the weights and biases on the $i$-th layer.
Moreover, we denote the parameters obtained at initialization by $\vbf\!\in\!\R^d$.
Thanks to this additional notation, we can now define $6$ parametric functions $\mu$ associated with $6$ Gibbs distributions (and bounds).
We consider regularized empirical risks with the optimizable norms studied by~\citet[Sec.C]{jiang2020fantastic} and defined as follows:

\raisebox{0.5ex}{\tiny\textbullet}~$\riskdistfro_{\beta}(h, \Scal) = {\footnotesize \alpha(\beta\RiskLossp_{\Scal}(h) + \bar{\beta}\distfro(h, \Scal))},$\\ 
\raisebox{0.5ex}{\tiny\textbullet}~$\riskdistltwo_{\beta}(h, \Scal) = {\footnotesize \alpha(\beta\RiskLossp_{\Scal}(h) + \bar{\beta}\distltwo(h, \Scal))},$\\ 
\raisebox{0.5ex}{\tiny\textbullet}~\mbox{$\riskparamnorm_{\beta}(h, \Scal) = {\footnotesize \alpha(\beta\RiskLossp_{\Scal}(h) + \bar{\beta}\paramnorm(h, \Scal))},$}\\ 
\raisebox{0.5ex}{\tiny\textbullet}~\mbox{$\riskpathnorm_{\beta}(h, \Scal)= {\footnotesize \alpha(\beta\RiskLossp_{\Scal}(h){+}\bar{\beta}\pathnorm(h, \Scal))},$}\\ 
\raisebox{0.5ex}{\tiny\textbullet}~$\risksumfro_{\beta}(h, \Scal) = {\footnotesize \alpha(\beta\RiskLossp_{\Scal}(h) + \bar{\beta}\sumfro(h, \Scal))},$\\
\raisebox{0.5ex}{\tiny\textbullet}~$\riskgap_{\beta}(h, \Scal) = {\footnotesize \alpha(\beta\RiskLossp_{\Scal}(h) + \bar{\beta}\gap(h, \Scal))}$,

where $\bar{\beta}=1-\beta$ and with 

\raisebox{0.5ex}{\tiny\textbullet}~$\distfro(h, \Scal) = \sum_{i=1}^{L} \|\wbf_{i}{-}\vbf_{i}\|_{2},$\\ 
\raisebox{0.5ex}{\tiny\textbullet}~$\distltwo(h, \Scal) = {\footnotesize \|\wbf{-}\vbf\|_{2}},$\\ 
\raisebox{0.5ex}{\tiny\textbullet}~$\paramnorm(h, \Scal) = {\footnotesize \sum_{i=1}^{L}\|\wbf_{i}\|_{2}^{2}},$\\ 
\raisebox{0.5ex}{\tiny\textbullet}~\mbox{$\pathnorm(h, \Scal)= {\footnotesize \sum_{y\in\Ycal} h_{\wbf^2}(\onebf)[y]},$}\\ 
\raisebox{0.5ex}{\tiny\textbullet}~$\sumfro(h, \Scal) = {\footnotesize L{[\prod_{i=1}^{L}\|\wbf_{i}\|_{2}^{2}]^{\frac1L}}},$\\
\raisebox{0.5ex}{\tiny\textbullet}~$\gap(h, \Scal) = {\footnotesize \vert\RiskLossp_{\Tcal}(h){-}\RiskLossp_{\Scal}(h)\vert}$.

Remark that $\gap$ is not a function that can be computed in practice, however, it can be interpreted as an ideal case when we want the norm to be representative of the gap and check if it correlates with it as done by \citet{jiang2019predicting}.
Note that during SGLD, we do not evaluate the gap on the whole learning samples $\Scal$ and $\Tcal$ but on a mini-batch of size $64$.
Moreover, by taking into account the \citet{dziugaite2018data}'s bounded cross-entropy instead of the classical one allows the parametric function not to focus too much on the risk since we want to take into account the norm.
We consider the bounds of \Cref{corollary:disintegrated-comp-unif,corollary:disintegrated-comp-data} with a split ratio of \mbox{$\frac{m'}{m+m'}\!=\!0.5$} and \mbox{$\frac{m'}{m+m'}\!=\!0.0$} while we set the parametric function $\omega$ associated with $\P$ to \mbox{$\omega(h_\wbf)\!=\!\comp(h_\wbf, \Scal')$} (the same function as for $\AQ$).
We report in \Cref{fig:reg-risk} the evolution of the test risks $\Risk_{\Tcal}(h)$ and the bound values for the different parametric functions as a function of the trade-off \mbox{parameter $\beta$}.

\textbf{Analysis.}
The main striking result is that the $\gap$ behaves differently than the norms.
Its test risk rapidly decreases (until $\beta\!=\!0.3$) while the associated bound remains tight. 
In contrast, the norms' curves show two regimes depending on the split ratio $\frac{m'}{m+m'}$.
For instance, when $\frac{m'}{m+m'}\!=\!0.0$, the test risks and the bounds decrease when $\beta\!\ge\!0.7$ but their gap increases.
When $\frac{m'}{m+m'}\!=\!0.5$, the test risks decrease for $\beta\!\ge\!0.7$ and the bounds stay tight.
Note that the bounds and test risks for $\paramnorm$ stay high because SGLD fails to minimize the regularized risk.
This experiment suggests that the norms are not good predictors of the generalization gap, as the norms' bounds are not close to the ideal one, given by $\gap$. 

\subsection{Experiments on Neural Complexities}
\label{sec:experiments-neural-comp}
In light of the previous results, we now study how our bounds behave when computed with a better predictor of the generalization gap.
Indeed, while \Cref{sec:experiments-emp-risk,sec:experiments-reg-risk} focus on hypotheses that minimize (regularized) empirical risks, we are ideally interested in concentrating the probability measure associated with $\AQ$ on the hypotheses with small generalization gaps.
To do so, the parametric function for $\AQ$ can depend on an estimation of the gap (this latter being not available in practice). 
In this section, we consider the bound of \Cref{corollary:disintegrated-comp-unif} (without data-dependent priors) and study the following parametric functions $\mu$
\begin{align*}
\mu(h,\Scal) = f^{\boldsymbol{D}}(h,\Scal) = \alpha\vert f(h,\Scal) - f(h_{\text{SGD}},\Scal)\vert,
\end{align*}
where $f\! \in\! \{\distfro, \distltwo, \paramnorm, \pathnorm,$ $\sumfro, \gap\}$, and $\alpha\!=\!m$, and  $h_{\text{SGD}}$ is obtained by Stochastic Gradient Descent (SGD).
This particular choice of $\mu$ allows us to sample hypotheses close to the value of the parametric function $f$ evaluated on $h_{\text{SGD}}$.
We additionally assess a parametric function $\distneural\!$, consisting of a neural network learned to predict the generalization gap.
More precisely, we learn the function $\neural(h,\Scal)$, which becomes the output of a feed-forward neural network (learned from $\Scal$), taking the parameters $\wbf$ of the model $h$ and outputting a positive real that must represent the generalization gap.
$\distneural$ is thus the function comparing the output of the feed-forward neural network associated with $h\sim\AQ$ and $h_{\text{SGD}}$.
Note that learning a neural network for predicting the gap was previously proposed by~\citep{lee2020neural}; we refer the reader to \Cref{sec:additional-experiments-neural-comp} for a discussion and a detailed presentation of the learning setting.
To obtain the model, we first run SGD on a random number of epochs uniformly sampled between $1$ and $10$, and with the same parameters as SGLD (\Cref{sec:experiments-setting}).
To sample $h\!\sim\!\AQ$, we start from $h_{\text{SGD}}$ as the initialization, and we run SGLD for $10$ epochs (unless the learning rate attains $10^{-10}$).
Finally, we consider a second setting, with the parametric functions are noted $\distfro$, $\distltwo$, $\paramnorm$, $\pathnorm$, $\sumfro$, $\gap$ (scaled by $\alpha\!=\!m$) and $\neural$ (without $\boldsymbol{D}$), where the parametric function $\comp(h, \Scal)$ is evaluated \textit{w.r.t.} the initial value of $h$ instead of $h_{\text{SGD}}$.
Note that for $\gap$ and $\distgap$, we skip the SGLD phase to have a bound on the model obtained from SGD directly; these two parametric functions correspond to our ideal cases.
We plot in \Cref{fig:neural-comp} the mean and the standard deviation of the bound values and test risks (averaged over 5 runs) for the considered parametric functions.
We provide additional experiments on $\neural$ and $\distneural$ in \Cref{sec:additional-experiments-neural-comp}.

\WarningFilter{latex}{Text page 8 contains only floats}
\begin{figure*}[!ht]
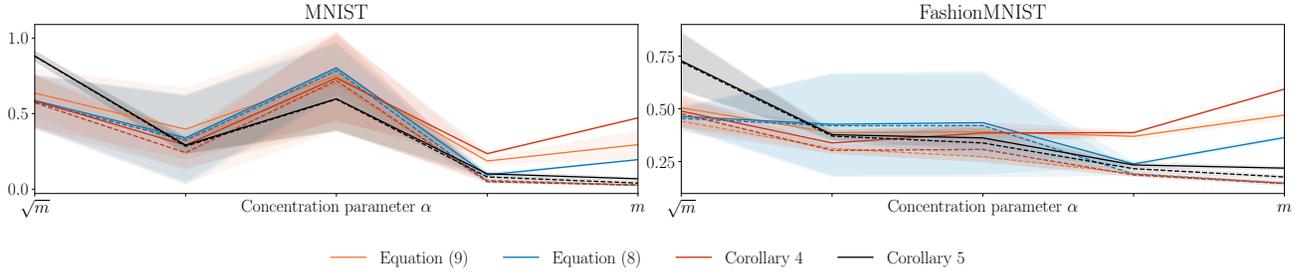

    \centering
    \includestandalone[width=1.0\linewidth]{figures/fig_2}
    \caption{Evolution of the bounds (the plain lines) and the test risks $\RiskLoss_{\Tcal}(h)$ (the dashed lines) {\it w.r.t.} the concentration parameter $\alpha$.
    The lines correspond to the mean, while the bands are the standard deviations.}
    \label{fig:emp-risk}
\end{figure*}

\begin{figure*}[!ht]
    \centering
    \includestandalone[width=1.0\linewidth]{figures/fig_3}
    \caption{Evolution of the bounds (the plain lines) and the test risks $\RiskLoss_{\Tcal}(h)$ (the dashed lines) \wrt the trade-off parameter $\beta$ for $\alpha=m$.
    The lines correspond to the mean, while the bands are the standard deviations.}
    \label{fig:reg-risk}
\end{figure*}

\begin{figure*}[!ht]
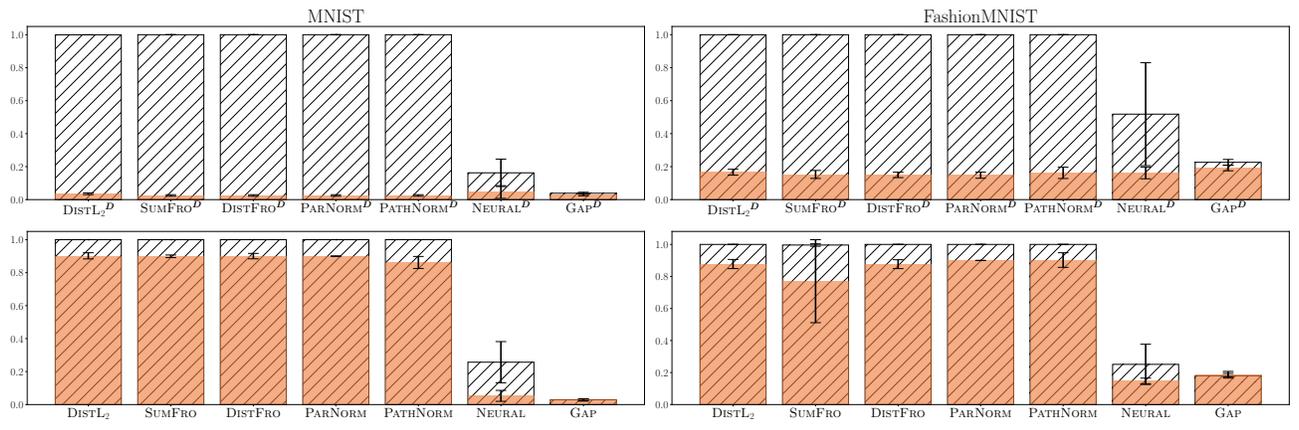

    \centering
    \includestandalone[width=1.0\linewidth]{figures/fig_4}
    \caption{Bar plot of the bound value associated with \Cref{corollary:disintegrated-comp-unif} and the different parametric functions. 
    The mean bound values of the sampled hypotheses $h\sim\AQ$ are shown with the hatched bars, and the mean test risks $\RiskLoss_{\Tcal}(h)$ are plotted in the colored bars. Moreover, the standard deviations are plotted in black.}
    \label{fig:neural-comp}
\end{figure*}

\textbf{Analysis.}
As a first general remark, the bounds with complexity measures based on norms behave differently than the ones based on $\neural$, $\distneural$, $\gap$, and $\distgap$.
Indeed, the mean bound values for the complexity measures based on the norms are all vacuous (\ie, they are equal to $1$), while their test risks are high for $\distfro$, $\distltwo$, $\paramnorm$, $\pathnorm$, and $\sumfro$, which is expected since we want to sample a hypothesis with a low norm (and not far from the initialization).
Similarly, for $\distdistfro$, $\distdistltwo$, $\distparamnorm$, $\distpathnorm$, and $\distsumfro$, we want to sample a hypothesis with a norm that is close to the one of $h_{\text{SGD}}$ and thus a hypothesis with a test risk $\RiskLoss_{\Tcal}(h)$ close to $\RiskLoss_{\Tcal}(h_{\text{SGD}})$.
In these cases, the bounds are vacuous because the parametric functions evaluated on $h$ are close to zero, and the ones evaluated on $h'$ are high.
This highlights a drawback of the empirical studies of \citet{jiang2020fantastic,dziugaite2020search}: they study the correlation between the norms and the generalization gaps on {\it trained} neural networks.
However, considering a norm as a good proxy for the generalization gap is impossible in this case. 
Indeed, rescaling the weights of the networks (by a scalar) gives the exact same predictions and thus keeps the same generalization gap while changing the norm; this is due to the use of non-negative homogeneous activation functions, such as the standard (Leaky) RELU~(see \eg, \citep{neyshabur2015path,dinh2017sharp}.
In contrast, the two parametric functions $\neural$ and $\distneural$ give tight bounds and are close to the ideal bounds of $\gap$ and $\distgap$.\footnote{For $\gap$, the bounds are sometimes lower than the test risks.
This is normal if the gap of $h_{\text{SGD}}$ is much higher than $0$ because sampling $h_{\text{SGD}}$ is unlikely in this context.}
This clearly illustrates that learning a parametric function (and so a complexity measure) can help to obtain tighter generalization bounds.
Note that the bounds with $\neural$ and $\distneural$ are tight even without a data-dependent prior, which is usually needed to obtain tight bounds for neural networks (see \eg, \citep{dziugaite2017computing,dziugaite2021role,perezortiz2021tighter,viallard2024general}).
This is an encouraging result and a step toward eliminating the need for data-dependent priors in PAC-Bayes to obtain tight bounds for neural networks.

\section{CONCLUSION}
\label{sec:conclu}

In contrast to classical statistical learning theory frameworks, for which a complexity measure is imposed, we provide a generic and novel generalization bound where the user can choose any parametric function acting as a complexity.
This measure incorporates a data and model-dependent function, which can be devised to favor desired properties for the hypotheses.
In particular, we show that when such a function is learned to be representative of the generalization gap, our bounds are tight even without data-dependent priors.
To the best of our knowledge, our framework is one of the few general enough to bring theoretical guarantees for learned complexity measures and for ones used in practice, 
\eg, based on some weight norms.
Last but not least, we believe this work paves the way for new research directions on bridging the gap between statistical learning theory and practice.
Indeed, our framework could provide meaningful insights into the generalization of deep models by plugging in new complexity measures, e.g., given by: {\it (i)} learning an interpretable model based on features such as training trajectory and network configuration, or {\it (ii)} new handcrafted parametric functions that are simple but predictive of generalization.

\subsubsection*{Acknowledgements}
This work was partially funded by the French ANR projects APRIORI ANR-18-CE23-0015, ANR-19-P3IA-0001 (PRAIRIE 3IA Institute), and FAMOUS ANR-23-CE23-0019.

\bibliography{main}
\bibliographystyle{plainnat}

\section*{Checklist}

 \begin{enumerate}

 \item For all models and algorithms presented, check if you include:
 \begin{enumerate}
   \item A clear description of the mathematical setting, assumptions, algorithm, and/or model. [Yes]
   \item An analysis of the properties and complexity (time, space, sample size) of any algorithm. [Not Applicable]
   \item (Optional) Anonymized source code, with specification of all dependencies, including external libraries. [Yes]
 \end{enumerate}

 \item For any theoretical claim, check if you include:
 \begin{enumerate}
   \item Statements of the full set of assumptions of all theoretical results. [Yes]
   \item Complete proofs of all theoretical results. [Yes]
   \item Clear explanations of any assumptions. [Yes]     
 \end{enumerate}

 \item For all figures and tables that present empirical results, check if you include:
 \begin{enumerate}
   \item The code, data, and instructions needed to reproduce the main experimental results (either in the supplemental material or as a URL). [Yes]
   \item All the training details (e.g., data splits, hyperparameters, how they were chosen). [Yes]
         \item A clear definition of the specific measure or statistics and error bars (e.g., with respect to the random seed after running experiments multiple times). [Yes]
         \item A description of the computing infrastructure used. (e.g., type of GPUs, internal cluster, or cloud provider). [No]
 \end{enumerate}

 \item If you are using existing assets (e.g., code, data, models) or curating/releasing new assets, check if you include:
 \begin{enumerate}
   \item Citations of the creator If your work uses existing assets. [Yes]
   \item The license information of the assets, if applicable. [Not Applicable]
   \item New assets either in the supplemental material or as a URL, if applicable. [Not Applicable]
   \item Information about consent from data providers/curators. [Not Applicable]
   \item Discussion of sensible content if applicable, e.g., personally identifiable information or offensive content. [Not Applicable]
 \end{enumerate}

 \item If you used crowdsourcing or conducted research with human subjects, check if you include:
 \begin{enumerate}
   \item The full text of instructions given to participants and screenshots. [Not Applicable]
   \item Descriptions of potential participant risks, with links to Institutional Review Board (IRB) approvals if applicable. [Not Applicable]
   \item The estimated hourly wage paid to participants and the total amount spent on participant compensation. [Not Applicable]
 \end{enumerate}

 \end{enumerate}

\clearpage


\appendix
\notappendixfalse
\onecolumn

The appendix is organized as follows:\begin{enumerate}[label={\it (\roman*)}]
    \item \Cref{sec:proof} is dedicated to the proof of \Cref{theorem:disintegrated-comp} (in \Cref{sec:proof-disintegrated-comp}), and to the proof of \Cref{corollary:disintegrated-comp-unif,corollary:disintegrated-comp-data} (in \Cref{sec:proof-corollary-disintegrated-comp}),
    \item \Cref{sec:comparison-literature} is dedicated to the theoretical results related to the comparison with other theoretical results of the literature,
    \item In \Cref{sec:comparison-uc-algo}, we explain how to obtain uniform-convergence and algorithmic-dependent bounds by setting appropriately the parametric function,
    \item Additional details on the experiments are provided in \Cref{sec:additional-experiments}.
\end{enumerate}

\section{PROOF OF THE MAIN RESULTS}
\label{sec:proof}

This section is dedicated to the proof of the results.
More precisely, in \Cref{sec:proof-disintegrated-comp}, we provide the proof of \Cref{theorem:disintegrated-comp}, whereas in \Cref{sec:proof-corollary-disintegrated-comp}, we prove \Cref{corollary:disintegrated-comp-unif,corollary:disintegrated-comp-data}.

\subsection[Proof of Theorem 3]{Proof of \Cref{theorem:disintegrated-comp}}
\label{sec:proof-disintegrated-comp}
\theoremboundcomplexitymeasure*
\begin{proof}
First of all, we denote as $Z=\int_{\Hcal}\exp\left[-\comp(g, \Scal)\right]d\lambda(g)$, the normalization constant of the Gibbs distribution $\AQ$ and $\lambda$ the reference measure on $\Hcal$. 
In other words, we have
\begin{align*}
    \AQ(h) = \frac{1}{Z}\exp\LB-\comp(h, \Scal)\RB \propto \exp\LB-\comp(h, \Scal)\RB.
\end{align*}
We apply \Cref{theorem:general-disintegrated-rivasplata} with $\frac{\delta}{2}$ instead of $\delta$ and with the function $\varphi(h,\Scal)=\phi(\RiskLoss_{\Dcal}(h), \RiskLoss_{\Scal}(h))$ to obtain
\begin{align*}
    \PP_{\Scal\sim\Dcal^m, h\sim\AQ}\LB\phi(\RiskLoss_{\Dcal}(h), \RiskLoss_{\Scal}(h)) \le \ln\!\LB\frac{\AQ(h)}{\P(h)}\RB\!+\!\ln\!\LB\frac{2}{\delta}\EE_{\Vcal\sim\Dcal^m}\EE_{g\sim\P}e^{\phi(\RiskLoss_{\Dcal}(g), \RiskLoss_{\Vcal}(g))}\RB \RB \ge 1-\frac{\delta}{2}.
\end{align*}
We develop the term $\ln\!\LB\frac{\AQ(h)}{\P(h)}\RB$ in \Cref{theorem:general-disintegrated-rivasplata}. 
We have 
\begin{align*} 
    \ln\!\LB\frac{\AQ(h)}{\P(h)}\RB &=  \ln\LP \frac{\exp\left[-\comp(h, \Scal)\right]}{Z}\frac{1}{\P(h)}\RP\\
    &= \ln\LP \exp\left[-\comp(h, \Scal)\right]\RP-\ln\LP\P(h)\int_{\Hcal}\!\exp\left[-\comp(g, \Scal)\right] d\lambda(g)\RP\\
    &= -\comp(h, \Scal)-\ln\LP\P(h)\int_{\Hcal}\frac{\P(g)}{\P(g)}  \exp\left[-\comp(g, \Scal)\right]d\lambda(g) \RP\\
    &= -\comp(h, \Scal) - \ln\LP\EE_{g\sim\P} \frac{\P(h)}{\P(g)}e^{-\comp(g, \Scal)}\RP.
\end{align*}

Hence, we obtain the following inequality
\begin{align}
    &\PP_{\Scal\sim\Dcal^m,h\sim\AQ}\Bigg[ \phi(\RiskLoss_{\Dcal}(h), \RiskLoss_{\Scal}(h)) \le \ln\!\LB\frac{2}{\delta}\EE_{\Vcal\sim\Dcal^m}\EE_{g\sim\P}e^{\phi(\RiskLoss_{\Dcal}(g), \RiskLoss_{\Vcal}(g))}\RB-\comp(h, \Scal) - \ln\LP{\textstyle\EE_{g\sim\P} \frac{\P(h)}{\P(g)}}e^{-\comp(g, \Scal)}\RP\Bigg] \ge 1-\frac{\delta}{2}.\label{eq:proof-theorem-disintegrated-comp-1}
\end{align}

We can now upper-bound the term $-\ln\LP{\textstyle\EE_{g\sim\P} \frac{\P(h)}{\P(g)}}e^{-\comp(g, \Scal)}\RP$.
To do so, since $\frac{\P(h)}{\P(h')}e^{-\comp(h', \Scal)} > 0$ for all $h\in\Hcal$, $h'\in\Hcal$ and $\Scal\in(\Xcal{\times}\Ycal)^{m}$, we apply Markov's inequality to obtain
\begin{align*}
    \forall h\in\Hcal,\quad \forall\Scal\in(\Xcal{\times}\Ycal)^{m},\quad \PP_{h'\sim\P}\Bigg[\frac{\P(h)}{\P(h')}e^{-\comp(h', \Scal)} \le \frac{2}{\delta}\EE_{g\sim\P}\frac{\P(h)}{\P(g)}e^{-\comp(g, \Scal)}\Bigg]\ge 1-\frac{\delta}{2}&\\
    \iff  \PP_{h'\sim\P}\Bigg[-\ln\LP\EE_{g\sim\P}\frac{\P(h)}{\P(g)}e^{-\comp(g, \Scal)}\RP \le \ln\frac{2}{\delta} -\ln\LP\frac{\P(h)}{\P(h')}e^{-\comp(h', \Scal)}\RP\Bigg]\ge 1-\frac{\delta}{2}.&
\end{align*}
Moreover, by simplifying the right-hand side of the inequality, we have 
\begin{align*}
    -\ln\LP\frac{\P(h)}{\P(h')}e^{-\comp(h'\!, \Scal)}\RP = \ln\frac{\P(h')}{\P(h)} +\comp(h'\!, \Scal).
\end{align*}

Hence, we obtain the following inequality
\begin{align}
    \PP_{h'\sim\P}\!\LB -\ln\!\LP\EE_{g\sim\P}\frac{\P(h)}{\P(g)}e^{-\comp(g, \Scal)}\RP \!\le \ln\frac{2}{\delta}{+}\ln\frac{\P(h')}{\P(h)}{+}\comp(h'\!, \Scal) \RB\! \ge 1{-}\frac{\delta}{2}.\label{eq:proof-theorem-disintegrated-comp-2}
\end{align}

By using a union bound on \Cref{eq:proof-theorem-disintegrated-comp-1,eq:proof-theorem-disintegrated-comp-2} and rearranging the terms, we obtain the claimed result.
\end{proof}

\subsection[Proof of Corollaries 4 and 5]{Proof of \Cref{corollary:disintegrated-comp-unif,corollary:disintegrated-comp-data}}
\label{sec:proof-corollary-disintegrated-comp}

In order to prove \Cref{corollary:disintegrated-comp-unif,corollary:disintegrated-comp-data}, we first provide another corollary that will be necessary.

\begin{restatable}{corollary}{corollarydisintegratedcomp}\label{corollary:disintegrated-comp}
For any $\Dcal$ on $\Xcal{\times}\Ycal$, for any hypothesis set $\Hcal$, for any loss $\loss : \Hcal\times(\Xcal{\times}\Ycal) \to [0,1]$, for any prior distribution $\P\in\Mcal(\Hcal)$ on $\Hcal$, for any $\comp\!:\! \Hcal{\times}(\Xcal{\times}\Ycal)^m{\to}\R$, for any $\delta\!\in\!(0,1]$, with probability at least $1{-}\delta$ over $\Scal{\sim}\Dcal^m$, $h'{\sim}\P$, $h{\sim}\AQ$ we have
\begin{align*}
\kl[\RiskLoss_{\Scal}(h) \|\RiskLoss_{\Dcal}(h)] &\le \frac{1}{m}\!\LB\comp(h'\!,\Scal) - \comp(h,\Scal) +\ln\frac{\P(h')}{\P(h)} + \ln\frac{8\sqrt{m}}{\delta^2}\RB.
\end{align*}
\end{restatable}
\begin{proof}
We instantiate \Cref{theorem:disintegrated-comp} with $\phi(\RiskLoss_{\Dcal}(h), \RiskLoss_{\Scal}(h)){=}m\kl[\RiskLoss_{\Scal}(h)\|\RiskLoss_{\Dcal}(h)]$.
By rearranging the term, we have 
\begin{align*}
\kl[\RiskLoss_{\Scal}(h) \|\RiskLoss_{\Dcal}(h)] &\le \frac{1}{m}\!\LB\comp(h'\!,\Scal) - \comp(h,\Scal) +\ln\frac{\P(h')}{\P(h)} + \ln\!\left[\frac{4}{\delta^2} \EE_{\Vcal\!{\sim}\Dcal^m}\EE_{g{\sim}\P} e^{m\kl[\RiskLoss_{\Scal}(g) \|\RiskLoss_{\Dcal}(g)]}\right]\RB.
\end{align*}
We also have to upper-bound $\EE_{\Vcal\sim\Dcal^m}\!\EE_{g\sim\P}\!\exp\LP m\kl\LB\RiskLoss_{\Vcal}(g) \|\RiskLoss_{\Dcal}(g)\RB\RP$. 
Indeed, we have
\begin{align}
    \EE_{\Vcal\sim\Dcal^m}\EE_{g\sim\P}e^{m\kl\LB\RiskLoss_{\Vcal}(g) \|\RiskLoss_{\Dcal}(g)\RB}  &=  \EE_{g\sim\P}\EE_{\Vcal\sim\Dcal^m}e^{m\kl\LB\RiskLoss_{\Vcal}(g) \|\RiskLoss_{\Dcal}(g)\RB}\label{eq:proof-corollary-disintegrated-comp-1}\\
    \text{and}\quad \EE_{g\sim\P}\EE_{\Vcal\sim\Dcal^m}e^{m\kl\LB\RiskLoss_{\Vcal}(g) \|\RiskLoss_{\Dcal}(g)\RB}  &\leq 2\sqrt{m}\label{eq:proof-corollary-disintegrated-comp-2},
\end{align}
where \Cref{eq:proof-corollary-disintegrated-comp-1} is due to Fubini's theorem (\ie, we can exchange the two expectations), and \Cref{eq:proof-corollary-disintegrated-comp-2} is due to \citet{maurer2004note}.
\end{proof}
Thanks to \Cref{corollary:disintegrated-comp}, we are now able to prove \Cref{corollary:disintegrated-comp-unif}.
\corollarydisintegratedcompunif*
\begin{proof}
We instantiate \Cref{corollary:disintegrated-comp} with $\P$ being a uniform distribution. 
Moreover, we have $\ln\frac{\P(h')}{\P(h)}=0$.  
\end{proof}
Similarly than for \Cref{corollary:disintegrated-comp-unif}, we prove \Cref{corollary:disintegrated-comp-data} thanks to \Cref{corollary:disintegrated-comp}.
\corollarydisintegratedcompdata*
\begin{proof}
We instantiate \Cref{corollary:disintegrated-comp} with $\P(h)\propto \exp(-\omega(h))$.
Moreover, we have $\ln\frac{\P(h')}{\P(h)}=\omega(h)-\omega(h')$.
\end{proof}

\section{COMPARISON WITH THE BOUNDS OF THE LITERATURE}
\label{sec:comparison-literature}

In this section, we first provide the bound of \citet{lee2020neural} in \Cref{sec:comparison-lee}.
Additionally, we discuss three bounds that are in the (classical or disintegrated) PAC-Bayesian literature and that consider Gibbs distributions.
More precisely, we discuss in \Cref{sec:comparison-catoni} a disintegrated bound of \citet[Section 1.2.4]{catoni2007pac} that was proven for a specific Gibbs distribution (\ie, with a fixed parametric function $\comp$).
Moreover, in \Cref{sec:comparison-lever,sec:comparison-dziugaite}, we provide two disintegrated PAC-Bayesian bounds for $\mu(h,\Scal)=\alpha\RiskLossp_{\Scal}(h)$ inspired by two works from the classical PAC-Bayesian literature.
Moreover, in \Cref{sec:related-works}, we discuss the related works.

\subsection[About Lee et al. {[2020]}’s Bound]{About \citet{lee2020neural}'s Bound}
\label{sec:comparison-lee}

For the sake of completeness, we provide a (refined) proof of the \citet{lee2020neural}'s bound.

\begin{theorem}

For any distribution $\Dcal$ on $\Xcal{\times}\Ycal$, for any hypothesis set $\Hcal$, for any distribution $\P\!\in\!\Mcal(\Hcal)$, for any $\delta\!\in\!(0, 1]$, with probability at least $1{-}\delta$ over $\Scal_1{\sim}\Dcal^m, \dots, \Scal_n{\sim}\Dcal^m$ and $h_1{\sim}\Q_{\Scal_1},\dots h_n{\sim}\Q_{\Scal_n}$ we have for all $\epsilon{>}0$
\begin{align*}
\PP_{\Scal\sim\Dcal^m,h\sim\AQ}\Big[ |\Risk_{\Dcal}(h)-\Risk_{\Scal}(h)| \le \comp(h, \Scal) + \epsilon \Big]\ge1{-}\frac{1}{n}\sum_{i=1}^{n}\indic\Big[ |\Risk_{\Dcal}(h_i){-}\Risk_{\Scal_i}(h_i)|{-}\comp(h_i, \Scal_i) > \epsilon \Big]{-}\sqrt{\frac{\ln\frac{2}{\delta}}{2n}} \defeq 1{-}\delta'(\epsilon),
\end{align*}
where $\indic[a]=1$ if $a$ is true and $0$ otherwise.
\end{theorem}
\begin{proof}
First of all, we have
\begin{align*}
\PP_{\Scal\sim\Dcal^m,h\sim\AQ}\Big[ |\Risk_{\Dcal}(h)-\Risk_{\Scal}(h)| \le \comp(h, \Scal) + \epsilon \Big] = \PP_{\Scal\sim\Dcal^m,h\sim\AQ}\Big[ |\Risk_{\Dcal}(h)-\Risk_{\Scal}(h)| -\comp(h, \Scal) \le \epsilon \Big] \defeq F(\epsilon),  
\end{align*}
where $F(\cdot)$ is the cumulative distribution function of $|\Risk_{\Dcal}(h){-}\Risk_{\Scal}(h)|{-}\comp(h, \Scal)$ where $\Scal\sim\Dcal^m$ and $h\sim\AQ$.
Then, from the Dvoretzky–Kiefer–Wolfowitz inequality, we have with probability at least $1-\delta$ over $\Scal_1{\sim}\Dcal^m, \dots, \Scal_n{\sim}\Dcal^m$ and $h_1{\sim}\Q_{\Scal_1},\dots h_n{\sim}\Q_{\Scal_n}$
\begin{align*}
F(\epsilon) \ge \frac{1}{n}\sum_{i=1}^{n}\indic\Big[ |\Risk_{\Dcal}(h_i){-}\Risk_{\Scal_i}(h_i)|{-}\comp(h_i, \Scal_i) \le \epsilon \Big] - \sqrt{\frac{\ln\frac{2}{\delta}}{2n}}.
\end{align*}
Moreover, remark that we have
\begin{align*}
\frac{1}{n}\sum_{i=1}^{n}\indic\Big[ |\Risk_{\Dcal}(h_i){-}\Risk_{\Scal_i}(h_i)|{-}\comp(h_i, \Scal_i) \le \epsilon \Big] &= \frac{1}{n}\sum_{i=1}^{n}\Bigg(1-\indic\Big[ |\Risk_{\Dcal}(h_i){-}\Risk_{\Scal_i}(h_i)|{-}\comp(h_i, \Scal_i) > \epsilon \Big]\Bigg)\\
&= 1-\frac{1}{n}\sum_{i=1}^{n}\indic\Big[ |\Risk_{\Dcal}(h_i){-}\Risk_{\Scal_i}(h_i)|{-}\comp(h_i, \Scal_i) > \epsilon \Big].
\end{align*}
Finally, combining the equations gives the desired result. 
\end{proof}

Note that we improve their result by replacing  $\Risk_{\Dcal}(h){-}\Risk_{\Scal}(h){-}\comp(h, \Scal)$ with $|\Risk_{\Dcal}(h){-}\Risk_{\Scal}(h)|{-}\comp(h, \Scal)$ in the empirical cumulative distribution function, which improves the constant in the statistical term $\sqrt{\nicefrac{\ln\frac{2}{\delta}}{2n}}$.
This does not change the interpretation of their results. 
Indeed, while the term $\comp(h, \Scal) {+} \epsilon$ is computable, the probability term $1{-}\delta'(\epsilon)$ is not since $\Risk_{\Dcal}(h_i)$ is unknown (as it depends on the data distribution $\Dcal$).
This makes the overall bound uncomputable as the probability $\delta'(\epsilon)$ by which it stands is unknown.

\subsection[About Catoni {[2007]}'s Bound]{About \citet{catoni2007pac}'s Bound}
\label{sec:comparison-catoni}

\citet[Theorem 1.2.7]{catoni2007pac} proved the following disintegrated PAC-Bayesian bound; we give a proof for the sake of completeness.
\begin{lemma}\label{theorem:disintegrated-catoni}
For any distribution $\Dcal$ on $\Xcal{\times}\Ycal$, for any hypothesis set $\Hcal$, for any distribution $\P\!\in\!\Mcal(\Hcal)$, for any $\delta\!\in\!(0, 1]$, we have with probability at least $1-\delta$ over $\Scal\sim\Dcal^m$ and $h\sim\AQ$
\begin{align*}
\RiskLoss_{\Dcal}(h) \le \frac{1}{1-e^{-c}}\LC1-\exp\LP -c\RiskLoss_{\Scal}(h) - \frac{1}{m}\LB \ln\frac{\AQ(h)}{\P(h)} + \ln\frac{1}{\delta}\RB\RP\RC.
\end{align*}
\end{lemma}
\begin{proof}
We apply \Cref{theorem:general-disintegrated-rivasplata} with $\varphi(h, \Scal)= m\LB -\ln(1{-}\RiskLoss_{\Dcal}(h)\LB 1{-}e^{-c}\RB) - c\RiskLoss_{\Scal}(h)\RB$.
By rearranging the terms, we obtain 
\begin{align}
\RiskLoss_{\Dcal}(h) \le \frac{1}{1-e^{-c}}\LC1-\exp\LP -c\RiskLoss_{\Scal}(h) - \frac{1}{m}\LB \ln\frac{\AQ(h)}{\P(h)} + \ln\LP\frac{1}{\delta}\EE_{\Scal\sim\Dcal^m}\EE_{g\sim\P}e^{m\LB -\ln(1{-}\RiskLoss_{\Dcal}(h)[1{-}e^{-c}])-c\RiskLoss_{\Scal}(h)\RB}\RP\RB\RP\RC.\label{eq:proof-disintegrated-catoni-1}
\end{align}
Moreover, from Fubini's theorem, \citet[Lemma 3]{maurer2004note}, and \citet[Corollary 2.2]{germain2009pac}, we have 
\begin{align}
\EE_{\Scal\sim\Dcal^m}\EE_{g\sim\P}e^{m\LB -\ln(1{-}\RiskLoss_{\Dcal}(h)[1{-}e^{-c}]) - c\RiskLoss_{\Scal}(h)\RB} \le 1.\label{eq:proof-disintegrated-catoni-2}
\end{align}
Finally, by merging \Cref{eq:proof-disintegrated-catoni-1,eq:proof-disintegrated-catoni-2}, we have the stated result.
\end{proof}
Compared to the bounds that we provided in this paper, this one depends on a parameter $c>0$ that is fixed before seeing the learning sample $\Scal\sim\Dcal^m$ and the hypothesis $h\sim\AQ$. 
\citeauthor{catoni2007pac} applied \Cref{theorem:disintegrated-catoni} for a particular Gibbs distribution.
In the following, we provide a more general corollary.
To obtain \citeauthor{catoni2007pac}'s corollary, we have to fix $\comp(h, \Scal)=cm\RiskLoss_{\Scal}(h)-\ln\P(h)$.
\begin{corollary}\label{corollary:disintegrated-catoni-original}
For any distribution $\Dcal$ on $\Xcal{\times}\Ycal$, for any hypothesis set $\Hcal$, for any loss $\loss : \Hcal\times(\Xcal{\times}\Ycal) \to [0,1]$, for any prior distribution $\P\in\Mcal(\Hcal)$ on $\Hcal$, for any parametric function $\comp\!:\! \Hcal{\times}(\Xcal{\times}\Ycal)^m{\to}\R$, for any $\delta\!\in\!(0,1]$, with probability at least $1{-}\delta$ over $\Scal{\sim}\Dcal^m$, $h{\sim}\AQ$ we have
\begin{align*}
\RiskLoss_{\Dcal}(h) \le \frac{1}{1-e^{-c}}\LC1-\exp\LP -c\RiskLoss_{\Scal}(h) - \frac{1}{m}\LB -\comp(h, \Scal) -\ln\P(h) - \ln\LP\EE_{g\sim\P} \frac{1}{\P(g)}e^{-\comp(g, \Scal)}\RP + \ln\frac{1}{\delta}\RB\RP\RC,
\end{align*}
where $\AQ$ is the Gibbs distribution (see \Cref{eq:gibbs-distribution}).
\end{corollary}
\begin{proof}
Starting from \Cref{theorem:disintegrated-catoni}, we develop the disintegrated KL divergence $\ln\frac{\AQ(h)}{\P(h)}$ as in \Cref{theorem:disintegrated-comp}.
We have
\begin{align*} 
    \ln\frac{\AQ(h)}{\P(h)} &= -\comp(h, \Scal) - \ln\LP\EE_{g\sim\P} \frac{\P(h)}{\P(g)}e^{-\comp(g, \Scal)}\RP\\
    &= -\comp(h, \Scal) -\ln\P(h) - \ln\LP\EE_{g\sim\P} \frac{1}{\P(g)}e^{-\comp(g, \Scal)}\RP,
\end{align*}
which leads to the desired result.
\end{proof}

In its current form, the generalization bound presented in \Cref{corollary:disintegrated-catoni-original} is not computable because of the expectation $\EE_{g\sim\P} \frac{1}{\P(g)}\exp[-\comp(g, \Scal)]$.
In order to obtain a term that is computable, we can do the same trick as in \Cref{theorem:disintegrated-comp}.
This gives the following bound.
\begin{corollary}\label{corollary:disintegrated-catoni}
For any distribution $\Dcal$ on $\Xcal{\times}\Ycal$, for any hypothesis set $\Hcal$, for any loss $\loss : \Hcal\times(\Xcal{\times}\Ycal) \to [0,1]$, for any prior distribution $\P\in\Mcal(\Hcal)$ on $\Hcal$, for any parametric function $\comp\!:\! \Hcal{\times}(\Xcal{\times}\Ycal)^m{\to}\R$, for any $\delta\!\in\!(0,1]$, with probability at least $1{-}\delta$ over $\Scal{\sim}\Dcal^m$, $h'{\sim}\P$, $h{\sim}\AQ$ we have
\begin{align*}
\RiskLoss_{\Dcal}(h) \le \frac{1}{1-e^{-c}}\LC1-\exp\LP -c\RiskLoss_{\Scal}(h) - \frac{1}{m}\!\LB\comp(h'\!,\Scal) - \comp(h,\Scal) +\ln\frac{\P(h')}{\P(h)} + \ln\frac{4}{\delta^2}\RB\RP\RC,
\end{align*}
where $\AQ$ is the Gibbs distribution (see \Cref{eq:gibbs-distribution}).
\end{corollary}
\begin{proof}
We first consider \Cref{corollary:disintegrated-catoni-original} with $\delta/2$ instead of $\delta$.
This gives the following bound
\begin{align}
\PP_{\Scal\sim\Dcal^m,h\sim\AQ}\LB\RiskLoss_{\Dcal}(h) \le \frac{1}{1{-}e^{-c}}\LC1{-}\exp\LP\!-c\RiskLoss_{\Scal}(h){-}\frac{1}{m}\LB-\!\comp(h, \Scal){-}\ln\LP\EE_{g\sim\P} \frac{\P(h)}{\P(g)}e^{-\comp(g, \Scal)}\RP{+}\ln\frac{2}{\delta}\RB\RP\RC\RB\!\ge1{-}\frac{\delta}{2}.\label{eq:proof-disintegrated-catoni-original-1}
\end{align}
Then, we use \Cref{eq:proof-theorem-disintegrated-comp-2}, which tells us that 
\begin{align*}
    \PP_{h'\sim\P}\!\LB -\ln\!\LP\EE_{g\sim\P}\frac{\P(h)}{\P(g)}e^{-\comp(g, \Scal)}\RP \!\le \ln\frac{2}{\delta}{+}\ln\frac{\P(h')}{\P(h)}{+}\comp(h'\!, \Scal) \RB\! \ge 1{-}\frac{\delta}{2}.
\end{align*}
Finally, combining \Cref{eq:proof-disintegrated-catoni-original-1} with \Cref{eq:proof-theorem-disintegrated-comp-2} gives us the desired result.
\end{proof}

As we can remark, the bound of \Cref{corollary:disintegrated-catoni} depends on the same terms as \Cref{corollary:disintegrated-comp}.
Hence, in order to compare \Cref{corollary:disintegrated-catoni} and \Cref{corollary:disintegrated-comp}, we prove the following proposition.

\begin{proposition}\label{proposition:disintegrated-catoni-comp}
For any distribution $\Dcal$ on $\Xcal{\times}\Ycal$, for any hypothesis set $\Hcal$, for any loss $\loss : \Hcal\times(\Xcal{\times}\Ycal) \to [0,1]$, for any prior distribution $\P\in\Mcal(\Hcal)$ on $\Hcal$, for any parametric function $\comp\!:\! \Hcal{\times}(\Xcal{\times}\Ycal)^m{\to}\R$, for any $\delta\!\in\!(0,1]$, with probability at least $1{-}\delta$ over $\Scal{\sim}\Dcal^m$, $h'{\sim}\P$, $h{\sim}\AQ$ we have
    \begin{align*}
    \RiskLoss_{\Dcal}(h) &\le \inf_{c>0}\LC\frac{1}{1-e^{-c}}\LC1-\exp\LP -c\RiskLoss_{\Scal}(h) - \frac{1}{m}\!\LB\comp(h'\!,\Scal) - \comp(h,\Scal) +\ln\frac{\P(h')}{\P(h)} + \ln\frac{8\sqrt{m}}{\delta^2}\RB\RP\RC\RC\\
    &= \underbrace{\klmax\LB\RiskLoss_{\Scal}(h) \;\middle|\; \frac{1}{m}\!\LP\comp(h'\!,\Scal) - \comp(h,\Scal) +\ln\frac{\P(h')}{\P(h)} + \ln\frac{8\sqrt{m}}{\delta^2}\RP\,\RB}_{\text{Bound of \Cref{corollary:disintegrated-comp}}},
    \end{align*}
where $\AQ$ is the Gibbs distribution (see \Cref{eq:gibbs-distribution}).
\end{proposition}
\begin{proof}
We apply the same proof of \citet{letarte2019dichotomize}'s Theorem 3 where in our case their ``$\Lcal_{\Dcal}(G_{\theta})$'', ``$\Lcal_{\Scal}(G_{\theta})$'' are respectively $\RiskLoss_{\Dcal}(h)$ and $\RiskLoss_{\Scal}(h)$ and ``$\xi$'' is defined by $\xi\defeq\frac{1}{m}\!\LP\comp(h'\!,\Scal) - \comp(h,\Scal) +\ln\frac{\P(h')}{\P(h)} + \ln\frac{8\sqrt{m}}{\delta^2}\RP$.
\end{proof}

In other words, the bound of \Cref{corollary:disintegrated-comp} is a \citeauthor{catoni2007pac}-like bound where the parameter $c>0$ is optimized.
At first sight, the bound in \Cref{corollary:disintegrated-catoni} might appear slightly tighter than \Cref{corollary:disintegrated-comp} (in light of \Cref{proposition:disintegrated-catoni-comp}).
Indeed, \Cref{corollary:disintegrated-comp}'s bound has an additional cost of $\frac{\ln(2\sqrt{m})}{m}$, which is negligible for a large number of examples $m$. 
However, the parameter $c>0$ in \Cref{corollary:disintegrated-catoni} cannot be optimized since the bound holds for a fixed parameter.
In order to optimize the bound, the bound must hold for a set of parameters $c$.
This can be done through the union bound (that adds an additional cost to the bound).
Hence, in order for \Cref{corollary:disintegrated-catoni} to be tighter, the additional cost cannot be larger than $\frac{\ln(2\sqrt{m})}{m}$, which is challenging for large $m$.
Hence, for the experiments, we did not consider the bound of \Cref{corollary:disintegrated-catoni}, which is only as tight as \Cref{corollary:disintegrated-comp} or larger.

\subsection[About Equation (8)]{About \Cref{eq:comparaison-lever}}
\label{sec:comparison-lever}

\citet[Lemma 5]{lever2013tighter} proved a (classical) PAC-Bayesian bound on the expected risk with $\comp(h,\Scal)\!=\!\alpha\RiskLossp_{\Scal}(h)$, \ie, they proved a bound on $\kl[\EE_{h\sim\AQ}\RiskLoss_{\Scal}(h)\|\EE_{h\sim\AQ}\RiskLoss_{\Dcal}(h)]$.
For the sake of comparison, we prove the following disintegrated bound that is similar to the one of \citet{lever2013tighter}.
We consider this bound as a baseline for the experiments in \Cref{sec:experiments}.

\begin{theorem}
\label{theorem:disintegrated-lever}
For any distribution $\Dcal$ on $\Xcal{\times}\Ycal$, for any hypothesis set $\Hcal$, for any losses $\loss : \Hcal\times(\Xcal{\times}\Ycal) \to [0,1]$ and $\loss' : \Hcal\times(\Xcal{\times}\Ycal) \to [0,1]$, for any $\delta\!\in\!(0,1]$, with probability at least $1{-}\delta$ over $\Scal{\sim}\Dcal^m$, $h{\sim}\AQ$ we have
\eqcomparisonlever
where the posterior $\AQ$ and the prior $\P$ are defined respectively by $\AQ(h) \propto e^{-\alpha\RiskLossp_{\Scal}(h)}$ and $\P(h) \propto e^{-\alpha\RiskLossp_{\Dcal}(h)}$.
\end{theorem}

Compared to the other bounds, \Cref{theorem:disintegrated-lever} does not depend on a parametric function $\comp$.
Instead, it depends only on the concentration parameter $\alpha\in\R$ and the number of examples $m$.
To obtain such a bound, the disintegrated KL divergence $\frac{\AQ(h)}{\P(h)}$ is upper-bounded. 
Hence, to prove \Cref{theorem:disintegrated-lever}, we first prove the following lemma (that is also inspired by \citeauthor{lever2013tighter}'s Lemma 4).

\begin{lemma}[Disintegrated version of \citeauthor{lever2013tighter}'s Lemma 4]\label{lemma:dis-kl-lever}
Given the posterior $\AQ$ and the prior $\P$ defined as $\AQ(h) \propto e^{-\comp(h, \Scal)}$ and $\P(h) \propto e^{-\omega(h)}$, we have the following upper-bound:
\begin{align*}
\forall h\in\Hcal,\quad \ln_{+}\!\frac{\AQ(h)}{\P(h)} \le \LB \omega(h)-\comp(h, \Scal) \RB_{+} +\LB\EE_{h'\sim\P} \comp(h', \Scal)-\omega(h') \RB_{+},
\end{align*}
where $[\cdot]_{+}\defeq\max(\cdot, 0)$ and $\ln_{+}(\cdot) \defeq [\ln(\cdot)]_{+}$.
\end{lemma}
\begin{proof}
First of all, we denote as $Z_{\AQ}=\int_{\Hcal}\exp\left[-\comp(g, \Scal)\right]d\lambda(g)$ and $Z_{\P}=\int_{\Hcal}\exp\left[-\omega(g)\right]d\lambda(g)$, the normalization constant of the Gibbs distributions $\AQ$ and $\P$ respectively while $\lambda$ is the reference measure on $\Hcal$.
Then, we have
\begin{align}
\ln_{+}\!\frac{\AQ(h)}{\P(h)} &= \ln_{+}\!\frac{Z_{\P}e^{-\comp(h, \Scal)}}{Z_{\AQ}e^{-\omega(h)}}\nonumber\\
&\le \LB \omega(h)-\comp(h, \Scal) \RB_{+} + \ln_{+}\frac{Z_{\P}}{Z_{\AQ}}\label{eq:proof-dis-kl-lever-1}\\
&= \LB \omega(h)-\comp(h, \Scal) \RB_{+} + \max\LP \ln\frac{Z_{\P}}{Z_{\AQ}}, 0\RP\nonumber\\
&= \LB \omega(h)-\comp(h, \Scal) \RB_{+} + \max\LP - \ln\LP\frac{1}{Z_{\P}}\int_{\Hcal}e^{-\comp(g, \Scal)}d\lambda(g)\RP, 0\RP\nonumber\\
&= \LB \omega(h)-\comp(h, \Scal) \RB_{+} + \max\LP - \ln\LP\frac{1}{Z_{\P}}\int_{\Hcal}e^{\omega(g)}e^{-\omega(g)}e^{-\comp(g, \Scal)}d\lambda(g)\RP, 0\RP\nonumber\\
&= \LB \omega(h)-\comp(h, \Scal) \RB_{+} + \max\LP - \ln\LP\int_{\Hcal}\P(g)e^{\omega(g)-\comp(g, \Scal)}d\lambda(g)\RP, 0\RP\nonumber\\
&= \LB \omega(h)-\comp(h, \Scal) \RB_{+} + \max\LP - \ln\LP\EE_{h'\sim\P}e^{\omega(h')-\comp(h', \Scal)}\RP, 0\RP\nonumber\\
&\le \LB \omega(h)-\comp(h, \Scal) \RB_{+} + \max\LP - \EE_{h'\sim\P}\LB \omega(h')-\comp(h', \Scal) \RB, 0\RP\label{eq:proof-dis-kl-lever-2}\\
&= \LB \omega(h)-\comp(h, \Scal) \RB_{+} + \LB\EE_{h'\sim\P} \comp(h', \Scal)-\omega(h')\RB_{+},\nonumber
\end{align}
where \Cref{eq:proof-dis-kl-lever-1} is obtained thanks to the inequality $[a{+}b]_{+} \le [a]_{+}{+}[b]_{+}$ while \Cref{eq:proof-dis-kl-lever-2} holds thanks to Jensen's inequality and because $[\cdot]_{+}$ is monotonically increasing.
\end{proof}

Moreover, in order to prove \Cref{theorem:disintegrated-lever}, we need the following lemma, which is an application of \Cref{theorem:general-disintegrated-rivasplata} given by \citet{rivasplata2020pac}.
\begin{lemma}\label{lemma:general-disintegrated-rivasplata-seeger-mcallester}
For any $\Dcal$ on $\Xcal{\times}\Ycal$, for any hypothesis set $\Hcal$, for any loss $\loss : \Hcal\times(\Xcal{\times}\Ycal) \to [0,1]$, for any prior distribution $\P\in\Mcal(\Hcal)$ on $\Hcal$, for any $\delta\!\in\!(0,1]$, with probability at least $1{-}\delta$ over $\Scal{\sim}\Dcal^m$, $h{\sim}\AQ$ we have
\begin{align}
&\kl[\RiskLoss_{\Scal}(h) \|\RiskLoss_{\Dcal}(h)] \le \frac{1}{m}\!\LB \ln_{+}\!\frac{\AQ(h)}{\P(h)} + \ln\frac{2\sqrt{m}}{\delta}\RB,\label{eq:general-disintegrated-rivasplata-seeger}\\
\text{and}\quad&\LN\RiskLoss_{\Scal}(h)-\RiskLoss_{\Dcal}(h)\RN \le \sqrt{\frac{1}{2m}\!\LB \ln_{+}\!\frac{\AQ(h)}{\P(h)} + \ln\frac{2\sqrt{m}}{\delta}\RB},\label{eq:general-disintegrated-rivasplata-mcallester}
\end{align}
where $\AQ\in\Mcal(\Hcal)$ is a posterior distribution.
\end{lemma}
\begin{proof}
We apply \Cref{theorem:general-disintegrated-rivasplata} with $\varphi(h,\Scal)=m\kl[\RiskLoss_{\Scal}(h) \|\RiskLoss_{\Dcal}(h)]$ to obtain 
\begin{align*}
\kl[\RiskLoss_{\Scal}(h) \|\RiskLoss_{\Dcal}(h)] &\le \frac{1}{m}\!\LB \ln\frac{\AQ(h)}{\P(h)} + \ln\!\left[\frac{1}{\delta} \EE_{\Vcal\!{\sim}\Dcal^m}\EE_{g{\sim}\P} e^{m\kl[\RiskLoss_{\Scal}(g) \|\RiskLoss_{\Dcal}(g)]}\right]\RB.
\end{align*}
From Fubini's theorem and \citet{maurer2004note}, we have
\begin{align}
\EE_{\Vcal\sim\Dcal^m}\EE_{g\sim\P}e^{m\kl\LB\RiskLoss_{\Vcal}(g) \|\RiskLoss_{\Dcal}(g)\RB} =  \EE_{g\sim\P}\EE_{\Vcal\sim\Dcal^m}e^{m\kl\LB\RiskLoss_{\Vcal}(g) \|\RiskLoss_{\Dcal}(g)\RB} \leq 2\sqrt{m}.
\end{align}
By definition of $\ln_{+}(\cdot)$, we have $\ln\frac{\AQ(h)}{\P(h)} \le \ln_{+}\frac{\AQ(h)}{\P(h)}$, which is \Cref{eq:general-disintegrated-rivasplata-seeger}.
Finally, thanks to Pinsker's inequality, we have $2(\RiskLoss_{\Scal}(h)-\RiskLoss_{\Dcal}(h))^2 \le \kl[\RiskLoss_{\Scal}(h) \|\RiskLoss_{\Dcal}(h)]$ and 
we obtain \Cref{eq:general-disintegrated-rivasplata-mcallester}.
\end{proof}

Thanks to \Cref{lemma:dis-kl-lever,lemma:general-disintegrated-rivasplata-seeger-mcallester}, we are now able to prove \Cref{theorem:disintegrated-lever}.
\begin{proof}[Proof of \Cref{theorem:disintegrated-lever}]
Starting from \Cref{lemma:general-disintegrated-rivasplata-seeger-mcallester} (and \Cref{eq:general-disintegrated-rivasplata-seeger}) with probability at least $1-\delta/3$ instead of $1-\delta$, we have
\begin{align}
\kl[\RiskLoss_{\Scal}(h) \|\RiskLoss_{\Dcal}(h)] &\le \frac{1}{m}\!\LB \ln_{+}\!\frac{\AQ(h)}{\P(h)} + \ln\!\frac{6\sqrt{m}}{\delta}\RB.\label{eq:proof-dis-kl-lever-risk-1}
\end{align}
From \Cref{lemma:dis-kl-lever}, we have
\begin{align}
\ln_{+}\!\frac{\AQ(h)}{\P(h)} \le \alpha\LB\RiskLossp_{\Dcal}(h)-\RiskLossp_{\Scal}(h)\RB_{+} + \alpha\LB\EE_{h'\sim\P} \RiskLossp_{\Scal}(h')-\RiskLossp_{\Dcal}(h') \RB_{+}.\label{eq:proof-dis-kl-lever-risk-2}
\end{align}
From \citep[Equation (4)]{maurer2004note} and Pinsker's inequality, we have with probability at least $1-\delta/3$ over $\Scal\sim\Dcal^m$
\begin{align}
\alpha\LB \EE_{h'\sim\P}\RiskLossp_{\Scal}(h')-\RiskLossp_{\Dcal}(h')\RB_{+} \le \alpha\LN\EE_{h'\sim\P}\RiskLossp_{\Scal}(h')-\RiskLossp_{\Dcal}(h')\RN \le \sqrt{\frac{\alpha^2}{2m}\ln\frac{6\sqrt{m}}{\delta}}.\label{eq:proof-dis-kl-lever-risk-3}
\end{align}
Moreover, from \Cref{lemma:general-disintegrated-rivasplata-seeger-mcallester} (and \Cref{eq:general-disintegrated-rivasplata-mcallester}), we can obtain with probability at least $1-\delta/3$ over $\Scal\sim\Dcal^m$ and $h\sim\AQ$
\begin{align}
\alpha\LB\RiskLossp_{\Dcal}(h)-\RiskLossp_{\Scal}(h)\RB_{+} 
\le \alpha\LN\RiskLossp_{\Scal}(h)-\RiskLossp_{\Dcal}(h)\RN \le \sqrt{\frac{\alpha^2}{2m}\LB\ln_{+}\!\frac{\AQ(h)}{\P(h)} + \ln\frac{6\sqrt{m}}{\delta} \RB}.\label{eq:proof-dis-kl-lever-risk-4}
\end{align}
From combining \Cref{eq:proof-dis-kl-lever-risk-2,eq:proof-dis-kl-lever-risk-3} with a union bound, we have with probability at least $1-2\delta/3$ over $\Scal\sim\Dcal^m$ and $h\sim\AQ$
\begin{align*}
&\ln_{+}\!\frac{\AQ(h)}{\P(h)} \le \sqrt{\frac{\alpha^2}{2m}\LB\ln_{+}\!\frac{\AQ(h)}{\P(h)} + \ln\frac{6\sqrt{m}}{\delta}\RB} + \sqrt{\frac{\alpha^2}{2m}\ln\frac{6\sqrt{m}}{\delta}}\\
\iff &\ln_{+}\!\frac{\AQ(h)}{\P(h)} + \ln\frac{6\sqrt{m}}{\delta} - \ln\frac{6\sqrt{m}}{\delta} \le \sqrt{\frac{\alpha^2}{2m}\LB\ln_{+}\!\frac{\AQ(h)}{\P(h)} + \ln\frac{6\sqrt{m}}{\delta}\RB} + \sqrt{\frac{\alpha^2}{2m}\ln\frac{6\sqrt{m}}{\delta}}\\
\iff &\ln_{+}\!\frac{\AQ(h)}{\P(h)} + \ln\frac{6\sqrt{m}}{\delta} - \ln\frac{6\sqrt{m}}{\delta} - \sqrt{\frac{\alpha^2}{2m}\LB\ln_{+}\!\frac{\AQ(h)}{\P(h)} + \ln\frac{6\sqrt{m}}{\delta}\RB} - \sqrt{\frac{\alpha^2}{2m}\ln\frac{6\sqrt{m}}{\delta}} \le 0.
\end{align*}
We obtain the upper-bound on $\ln_{+}\!\frac{\AQ(h)}{\P(h)}$ by solving the quadratic (in)equation
\begin{align*}
ax^2+bx+c \le 0 \quad\text{such that}\quad x\in\R^{+} \quad\text{with}\quad a=1, \quad b=-\sqrt{\frac{\alpha^2}{2m}}, \quad\text{and}\quad c=-\ln\frac{6\sqrt{m}}{\delta}-\sqrt{\frac{\alpha^2}{2m}\ln\frac{6\sqrt{m}}{\delta}}.
\end{align*}
Hence solving the quadratic (in)equation gives 
\begin{align*}
x \in \LB 0, \sqrt{\frac{\alpha^2}{8m}+\ln\frac{6\sqrt{m}}{\delta}+ \sqrt{\frac{\alpha^2}{2m}\ln\frac{6\sqrt{m}}{\delta}}}-\sqrt{\frac{\alpha^2}{8m}}\RB.
\end{align*}
Hence, we can deduce that  
\begin{align*}
\sqrt{\ln_{+}\!\frac{\AQ(h)}{\P(h)} + \ln\frac{6\sqrt{m}}{\delta}} &\le \sqrt{\frac{\alpha^2}{8m}+\ln\frac{6\sqrt{m}}{\delta}+ \sqrt{\frac{\alpha^2}{2m}\ln\frac{6\sqrt{m}}{\delta}}}-\sqrt{\frac{\alpha^2}{8m}}\\
&\le \sqrt{\frac{\alpha^2}{8m}+\ln\frac{6\sqrt{m}}{\delta}+ \sqrt{\frac{\alpha^2}{2m}\ln\frac{6\sqrt{m}}{\delta}}}
\end{align*}
and
\begin{align}
\ln_{+}\!\frac{\AQ(h)}{\P(h)} \le \frac{\alpha^2}{8m} + \sqrt{\frac{\alpha^2}{2m}\ln\frac{6\sqrt{m}}{\delta}}.\label{eq:proof-dis-kl-lever-risk-5}
\end{align}
Combining \Cref{eq:proof-dis-kl-lever-risk-1,eq:proof-dis-kl-lever-risk-5} gives the desired result.
\end{proof}

Note that the proof technique differs from the one of \citet{lever2013tighter} because we have to use two disintegrated PAC-Bayesian bounds and one classical PAC-Bayesian bound instead of only one classical PAC-Bayesian bound.
Indeed, since the disintegrated bounds are valid only for one posterior distribution, we have to use one bound to obtain \Cref{eq:proof-dis-kl-lever-risk-1} and one, in \Cref{eq:proof-dis-kl-lever-risk-4}, that serves to upper-bound the disintegrated KL divergence.
The classical PAC-Bayesian bound in \Cref{eq:proof-dis-kl-lever-risk-3} serves to upper-bound the second term for the disintegrated KL divergence. 

\subsection[About Equation (9)]{About \Cref{eq:comparaison-dziugaite}}
\label{sec:comparison-dziugaite}

More recently, \citet[Theorem 4.2]{dziugaite2018data} proved a (classical) PAC-Bayesian bound on the expected risk with $\comp(h,\Scal)\!=\!\alpha\RiskLoss_{\Scal}(h)$ and considers data-dependent priors obtained from a $\epsilon$-differentially private mechanism.
However, their proof relies on the \textit{approximate max-information}~\citep{dwork2015generalization} that we cannot straightforwardly adapt to the disintegrated setting.
Instead, our proof is based on the definition of $\epsilon$-differential privacy (given by \citet[Section III]{mironov2017renyi}).

\begin{definition}\label{def:differentially-private}
A randomized mechanism $\P$ is $\epsilon$-differentially private if and only if for any learning samples $\Tcal'$ and $\Tcal$ differing from one example we have
\begin{align*}
D_{\infty}(\P_{\Tcal'}\|\P_{\Tcal}) \defeq \ln\LP\esssup_{h\sim \P_{\Tcal'}}\frac{\P_{\Tcal'}(h)}{\P_{\Tcal}(h)}\RP \le \epsilon.
\end{align*}
\end{definition}

Put into words, a randomized mechanism (\ie, the sampling from the prior $\P$) is $\epsilon$-differentially private if the ratio between the densities obtained from the two learning samples $\Tcal'$ and $\Tcal$ (differing from one point) is bounded by $\epsilon$.
Intuitively, the two densities must be close when the learning samples $\Tcal$ and $\Tcal'$ differ from only one example. 

From the definition, we are able to prove the following bound.
\begin{lemma}\label{lemma:general-disintegrated-differential-privacy}
For any distribution $\Dcal$ on $\Xcal{\times}\Ycal$, for any $\epsilon$-differentially private randomized mechanism $\P$, for any measurable function $\varphi: \Hcal\times(\Xcal{\times}\Ycal)^m\to \R$, for any $\delta\!\in\!(0, 1]$, we have with probability at least $1-\delta$ over $\Scal'\sim\Dcal^m$, $\Scal\sim\Dcal^m$ and $h\sim\AQ$
\begin{align*}
\varphi(h,\Scal) \le \ln\frac{\AQ(h)}{\P_{\Scal}(h)}+ m\epsilon + \ln\!\left[\frac{1}{\delta}\EE_{\Vcal\sim\Dcal^m}\EE_{g\sim\P_{\Scal'}}e^{\varphi(g,\Vcal)}\right].
\end{align*}
\end{lemma}
\begin{proof}
First of all, note that we can apply \Cref{theorem:general-disintegrated-rivasplata} with the data-dependent prior $\P_{\Scal'}$ depending on the ghost sample $\Scal'$.
Indeed, we have with probability at least $1-\delta$ over $\Scal'\sim\Dcal^m$, $\Scal\sim\Dcal^m$ and $h\sim\AQ$
\begin{align}
\varphi(h,\Scal) \le \ln\frac{\AQ(h)}{\P_{\Scal'}(h)}\!+\!\ln\!\left[\frac{1}{\delta}\EE_{\Vcal\sim\Dcal^m}\EE_{g\sim\P_{\Scal'}}e^{\varphi(g,\Vcal)}\right].\label{eq:proof-general-disintegrated-differential-privacy-1}
\end{align}
Let's denote by $\Scal^{(i)}$ the learning sample $\Scal$ such that the examples from index $1$ to $i$ have been replaced by the examples coming from the learning sample $\Scal'$.
By convention, we have thus $\Scal^{(0)}=\Scal$ and $\Scal^{(m)}=\Scal'$.
We can hence upper-bound the disintegrated KL divergence by 
\begin{align}
\ln\frac{\AQ(h)}{\P_{\Scal'}(h)} &= \ln\frac{\AQ(h)}{\P_{\Scal^{(m)}}(h)}\nonumber\\
&\le \ln\frac{\AQ(h)}{\P_{\Scal^{(m-1)}}(h)} + \epsilon\nonumber\\
&\dotsb\nonumber\\
&\le \ln\frac{\AQ(h)}{\P_{\Scal^{(0)}}(h)} + m\epsilon\nonumber\\
&= \ln\frac{\AQ(h)}{\P_{\Scal}(h)} + m\epsilon.\label{eq:proof-general-disintegrated-differential-privacy-2}
\end{align}
By combining \Cref{eq:proof-general-disintegrated-differential-privacy-1,eq:proof-general-disintegrated-differential-privacy-2}, we obtain the stated result.
\end{proof}

\Cref{lemma:general-disintegrated-differential-privacy} can be interpreted as a special case of \Cref{theorem:general-disintegrated-rivasplata} where $\P$ is a $\epsilon$-differentially private randomized mechanism.
Note that the bound is also in probability over $\Scal'\sim\Dcal^{m}$, which is a ghost sample (that we do not have in practice).
However, it is not problematic since \Cref{eq:comparaison-dziugaite} does not depend explicitly on $\Scal'\sim\Dcal^{m}$.
In order to prove further \Cref{eq:comparaison-dziugaite}, we now specialize \Cref{lemma:general-disintegrated-differential-privacy} to obtain a bound with a parametric function $\comp$ and a $\epsilon$-differentially private randomized mechanism $\P$.

\begin{theorem}\label{theorem:general-disintegrated-differential-privacy}
Let $\loss: \Hcal{\times}(\Xcal{\times}\Ycal){\to}\R$ be a loss function and $\phi\!:\! \R^2{\to}\R$ be a generalization gap.
For any distribution $\Dcal$ on $\Xcal{\times}\Ycal$, for any hypothesis set $\Hcal$, for any $\epsilon$-differentially private randomized mechanism $\P$, for any parametric function $\comp\!:\! \Hcal{\times}(\Xcal{\times}\Ycal)^m{\to}\R$, for any $\delta\!\in\!(0, 1]$, we have with probability at least $1-\delta$ over $\Scal'\sim\Dcal^m$, $\Scal\sim\Dcal^m$, $h'\sim\P_{\Scal}$ and $h\sim\AQ$
\begin{align*}
    \phi(\RiskLoss_{\Dcal}(h), \RiskLoss_{\Scal}(h)) &\le \comp(h'\!,\Scal) - \comp(h,\Scal) + \ln\frac{\P_{\Scal}(h')}{\P_{\Scal}(h)} + m\epsilon + \ln\!\left[\frac{4}{\delta^2} \EE_{\Vcal\!{\sim}\Dcal^m}\EE_{g{\sim}\P_{\Scal'}} e^{\phi(\RiskLoss_{\Dcal}(g),\RiskLoss_{\Vcal}(g))}\right],
\end{align*}
where $\AQ$ is the Gibbs distribution (see \Cref{eq:gibbs-distribution}).
\end{theorem}
\begin{proof}
Starting from \Cref{lemma:general-disintegrated-differential-privacy}, we follow the same steps as for \Cref{theorem:disintegrated-comp} to obtain the result.
Indeed, we first develop the term $\ln\frac{\AQ(h)}{\P_{\Scal}(h)}$ to have 
\begin{align*} 
    \ln\frac{\AQ(h)}{\P_{\Scal}(h)} &= -\comp(h, \Scal) - \ln\LP\EE_{g\sim\P_{\Scal}} \frac{\P_{\Scal}(h)}{\P_{\Scal}(g)}e^{-\comp(g, \Scal)}\RP.
\end{align*}
Hence, we obtain the following inequality
\begin{align}
    &\PP_{\Scal\sim\Dcal^m,h\sim\AQ}\!\Bigg[ \phi(\RiskLoss_{\Dcal}(h), \RiskLoss_{\Scal}(h)) \le \ln\!\LB\frac{2}{\delta}\EE_{\Vcal\sim\Dcal^m}\EE_{g\sim\P_{\Scal'}}e^{\phi(\RiskLoss_{\Dcal}(g), \RiskLoss_{\Vcal}(g))}\RB\!{-}\comp(h, \Scal) {-}\ln\!\LP{\textstyle\EE_{g\sim\P_{\Scal}} \frac{\P_{\Scal}(h)}{\P_{\Scal}(g)}}e^{-\comp(g, \Scal)}\RP\Bigg]\! \ge 1-\frac{\delta}{2}.\label{eq:proof-disintegrated-comp-1}
\end{align}
We can now upper-bound the term $-\ln\LP{\textstyle\EE_{g\sim\P_{\Scal}} \frac{\P_{\Scal}(h)}{\P_{\Scal}(g)}}e^{-\comp(g, \Scal)}\RP$.
To do so, since $\frac{\P_{\Scal}(h)}{\P_{\Scal}(h')}e^{-\comp(h', \Scal)} > 0$, we apply Markov's inequality to obtain
\begin{align*}
    \forall h\in\Hcal,\quad \forall\Scal\in(\Xcal{\times}\Ycal)^{m},\quad \PP_{h'\sim\P_{\Scal}}\Bigg[\frac{\P_{\Scal}(h)}{\P_{\Scal}(h')}e^{-\comp(h', \Scal)} \le \frac{2}{\delta}\EE_{g\sim\P_{\Scal}}\frac{\P_{\Scal}(h)}{\P(g)}e^{-\comp(g, \Scal)}\Bigg]\ge 1-\frac{\delta}{2}&\\
    \iff  \PP_{h'\sim\P_{\Scal}}\Bigg[-\ln\LP\EE_{g\sim\P_{\Scal}}\frac{\P_{\Scal}(h)}{\P_{\Scal}(g)}e^{-\comp(g, \Scal)}\RP \le \ln\frac{2}{\delta} -\ln\LP\frac{\P_{\Scal}(h)}{\P_{\Scal}(h')}e^{-\comp(h', \Scal)}\RP\Bigg]\ge 1{-}\delta.&
\end{align*}
Moreover, by simplifying the right-hand side of the inequality, we have 
\begin{align*}
    -\ln\LP\frac{\P_{\Scal}(h)}{\P_{\Scal}(h')}e^{-\comp(h'\!, \Scal)}\RP &= \ln\frac{\P_{\Scal}(h')}{\P_{\Scal}(h)} +\comp(h'\!, \Scal).
\end{align*}
Hence, we obtain the following inequality
\begin{align}
    \PP_{h'\sim\P_{\Scal}}\!\LB -\ln\!\LP\EE_{g\sim\P_{\Scal}}\frac{\P_{\Scal}(h)}{\P_{\Scal}(g)}e^{-\comp(g, \Scal)}\RP \!\le \ln\frac{2}{\delta}+\ln\frac{\P_{\Scal}(h')}{\P_{\Scal}(h)} +\comp(h'\!, \Scal) \RB\! \ge 1{-}\frac{\delta}{2}.\label{eq:proof-disintegrated-comp-2}
\end{align}
By using a union bound on \Cref{eq:proof-disintegrated-comp-1,eq:proof-disintegrated-comp-2} and rearranging the terms, we obtain the claimed result.
\end{proof}

We are now able to prove the bound stated in \Cref{eq:comparaison-dziugaite}.

\begin{corollary}\label{corollary:disintegrated-differential-privacy}
For any distribution $\Dcal$ on $\Xcal{\times}\Ycal$, for any hypothesis set $\Hcal$, for any losses $\loss: \Hcal{\times}(\Xcal{\times}\Ycal){\to}\R$ and $\loss': \Hcal{\times}(\Xcal{\times}\Ycal){\to}\R$, for any $\alpha,\alpha'\ge0$, for any $\delta\in(0, 1]$, we have with probability at least $1-\delta$ over $\Scal'\sim\Dcal^m$, $\Scal\sim\Dcal^m$, $h'\sim\P_{\Scal}$ and $h\sim\AQ$
\eqcomparisondziugaite
where the posterior $\AQ$ and the prior $\P$ are defined respectively by $\AQ(h) \propto e^{-\alpha\RiskLossp_{\Scal}(h)}$ and $\P(h) \propto e^{-\alpha'\RiskLossp_{\Scal}(h)}$.
\end{corollary}
\begin{proof}
We instantiate \Cref{theorem:general-disintegrated-differential-privacy}  with $\phi(\RiskLoss_{\Dcal}(h), \RiskLoss_{\Scal}(h))=m\kl[\RiskLoss_{\Scal}(h) \|\RiskLoss_{\Dcal}(h)]$. 
Additionally, from Fubini's theorem and \citet{maurer2004note} we have $\EE_{\Vcal\sim\Dcal^m}\EE_{g\sim\P}\exp(m\kl[\RiskLoss_{\Vcal}(g) \|\RiskLoss_{\Dcal}(g)]) \leq 2\sqrt{m}$.
Hence, we can deduce that
\begin{align*}
    \kl[\RiskLoss_{\Scal}(h) \|\RiskLoss_{\Dcal}(h)] \le \frac{1}{m}\LB\comp(h'\!,\Scal) - \comp(h,\Scal) + \ln\frac{\P_{\Scal}(h')}{\P_{\Scal}(h)} + m\epsilon + \ln\!\frac{8\sqrt{m}}{\delta^2}\RB.
\end{align*}
Let the posterior distribution $\AQ$ and the prior distribution $\P$ defined respectively by $\AQ\propto \exp(-\alpha\RiskLossp_{\Scal}(h))$ (\ie, $\comp(h,\Scal)=\alpha\RiskLossp_{\Scal}(h)$) and $\P_{\Scal}\propto \exp(-\alpha'\RiskLossp_{\Scal}(h))$.
From these definitions, we obtain
\begin{align}
    \kl[\RiskLoss_{\Scal}(h) \|\RiskLoss_{\Dcal}(h)] \le \frac{1}{m}\LB \alpha\!\LB\RiskLossp_{\Scal}(h'){-}\RiskLossp_{\Scal}(h)\RB + \alpha'\!\LB\RiskLossp_{\Scal}(h){-}\RiskLossp_{\Scal}(h')\RB + m\epsilon + \ln\!\frac{8\sqrt{m}}{\delta^2}\RB.\label{eq:proof-corollary:disintegrated-differential-privacy-1}
\end{align}
From \citet[Theorem 6]{mcsherry2007mechanism}, we can deduce that the randomized mechanism $\P_{\Scal}$ (\ie, the prior) gives $\epsilon=2\alpha'\frac{1}{m}$-differential privacy.
Hence, by simplifying the \Cref{eq:proof-corollary:disintegrated-differential-privacy-1}, we have the desired result.
\end{proof}

Even though \Cref{corollary:disintegrated-differential-privacy} does not use the approximate max-information as done by \citet{dziugaite2018data}, we are still able to provide a bound with a prior that gives a hypothesis $h'$ from an $\epsilon$-differentially private randomize mechanism.
The main advantage of these bounds compared to the others is that the prior can depend on the learning sample $\Scal$. 
This is why this bound is a good candidate for a baseline in \Cref{sec:experiments}.

\subsection{Related Works}
\label{sec:related-works}

The Gibbs distribution has been used in information-theoretic generalization bounds\footnote{See \citet{xu2017information,goyal2017pac,bu2020tightening} for some examples of information-theoretic bounds.} that upper-bound the expected generalization gap $\EE_{\Scal\sim\Dcal^m, h\sim\AQ}\RiskLoss_{\Dcal}(h){-}\RiskLoss_{\Scal}(h)$~.
For instance, \citet[Theorem 4,][]{raginsky2016information} provided  bounds for $\comp(h,\Scal){=}\alpha\RiskLoss_{\Scal}(h)$ with losses bounded between $0$ and $1$, while \citet[Theorem 1,][]{kuzborskij2019distribution} with sub-Gaussian losses.
\citet[Theorem 1,][]{aminian2021exact} proved a closed-form solution of the expected generalization gap with the Gibbs distribution with $\comp(h,\Scal){=}\alpha\RiskLoss_{\Scal}(h)$ (where the loss is non-negative); they also considered regularized empirical risks.
\citet{xu2017information,kuzborskij2019distribution} upper-bound the expected true risk $\EE_{\Scal\sim\Dcal^m, h\sim\AQ}\RiskLoss_{\Dcal}(h)$ by excess risk bounds, \ie, bounds \wrt the minimal true risk over the hypothesis set.
In the PAC-Bayesian literature, \citet{alquier2016properties} develop PAC-Bayesian bounds on the true risk $\EE_{h\sim\AQ}\RiskLoss_{\Dcal}(h)$ with $\comp(h,\Scal){=}\alpha\RiskLoss_{\Scal}(h)$.
However, all these bounds consider a (regularised) empirical risk scaled by $\alpha$ for the parametric function, while we are interested in user-defined parametric functions $\comp$.
Moreover, these bounds are in expectation over $h\sim\AQ$, while we are interested in the risk of a {\it single} hypothesis $h$ sampled from $\AQ$.
Hence, to the best of our knowledge, our contribution is the first to derive probabilistic bounds for a single hypothesis sampled from a Gibbs distribution with general parametric functions $\comp$.

\section{OBTAINING UNIFORM-CONVERGENCE AND ALGORITHMIC-DEPENDENT BOUNDS}
\label{sec:comparison-uc-algo}

\looseness=-1
In this section, we theoretically compare generalization bounds with arbitrary complexity measures and the literature's bounds.
To do so, we prove in \Cref{corollary:dis-uc,corollary:dis-algo} that, from an appropriate parametric function $\comp$, we can obtain two types of generalization bounds: the uniform-convergence-based and the algorithmic-dependent generalization bounds.
Hence, \Cref{corollary:dis-uc,corollary:dis-algo} do not present new uniform-convergence bounds but show how to obtain existing bounds by integrating a specific complexity measure.
In other words, we show that \Cref{theorem:disintegrated-comp} is general enough to obtain one bound belonging to one of these frameworks.
As we see in \Cref{sec:obtaining-uc,sec:obtaining-algo}, this is done by {\it (i)} assuming that we can effectively find an upper bound of the generalization gap and {\it (ii)} fixing the appropriate function $\comp$.
In order to present our results in \Cref{corollary:dis-uc,corollary:dis-algo}, we first recall the definitions of the literature's bounds.

\subsection{Types of Generalization Bounds in the Literature}

The uniform-convergence-based bounds were the first type to be introduced, notably in \citet{vapnik1971uniform} using the VC-dimension.
Other bounds were later developed based on the Gaussian/Rademacher complexity~\citep{bartlett2002rademacher}.
The definition of this type of bounds is the following.

\begin{restatable}[Uniform Convergence Bound]{definition}{definitionUC}\label{def:uc}
Let $\loss: \Hcal{\times}(\Xcal{\times}\Ycal){\to}\R$ be a loss function and $\phi\!:\! \R^2{\to}\R$ be a generalization gap.
A uniform convergence bound is defined such that if for any distribution $\Dcal$ on $\Xcal\times\Ycal$, for any hypothesis set $\Hcal$, there exists a function $\PhiUC: (0,1]{\to} \R$, such that for any  $\delta\in(0, 1]$ we have
\begin{align}
\PP_{\Scal\sim\Dcal^{m}}\Bigg[\, \forall h{\in}\Hcal,\ \phi(\RiskLoss_{\Dcal}(h),\RiskLoss_{\Scal}(h)) \le \PhiUC\big(\delta\big) \Bigg] = \PP_{\Scal\sim\Dcal^{m}}\LB\, \sup_{h\in\Hcal} \phi(\RiskLoss_{\Dcal}(h),\RiskLoss_{\Scal}(h)) \le \PhiUC\big(\delta\big) \RB\ge 1-\delta.\label{eq:uc}
\end{align}
\end{restatable}

\looseness=-1
This definition encompasses different complexity measures, such as $\PhiUC(\delta){=}\rad(\Hcal) {+} \sqrt{\frac{1}{2m}\ln\frac{1}{\delta}}$ for the Rademacher complexity $\rad(\Hcal)$, or $\PhiUC(\delta){=}\sqrt{\frac{1}{m}2\vc(\Hcal)\ln\frac{em}{\vc(\Hcal)}}{+}\sqrt{\tfrac{1}{2m}\ln\frac{1}{\delta}}$ for the VC-dimension $\vc(\Hcal)$~\citep[see Theorem 3.3 and Corollary 3.19 of ][]{mohri2012foundations} where the generalization gap is defined by $\phi(\RiskLoss_{\Dcal}(h),\RiskLoss_{\Scal}(h))=\RiskLoss_{\Dcal}(h)-\RiskLoss_{\Scal}(h)$ and $\loss$ is the 01-loss.
This definition also highlights the worst-case nature of the uniform-convergence bounds:
given a confidence $\delta$, the generalization gap $\phi(\RiskLoss_{\Dcal}(h), \RiskLoss_{\Scal}(h))$ is upper-bounded by a complexity measure $\PhiUC(\delta)$ constant for all $(h,\Scal)\!\in\!\Hcal{\times}(\Xcal{\times}\Ycal)^m$.
The upper bound $\PhiUC(\delta)$ can generally be improved by considering algorithmic-dependent bounds~\citep{bousquet2002stability,xu2012robustness}.
This kind of bounds upper-bound the generalization gap for the hypothesis $h_{\Scal}$ learned by an algorithm from a learning sample $\Scal$.
The definition of such bounds is recalled below.

\begin{restatable}[Algorithmic-dependent Generalization Bound]{definition}{definitionALGODEP}\label{def:algo} 
Let $\loss: \Hcal{\times}(\Xcal{\times}\Ycal){\to}\R$ be a loss function and $\phi\!:\! \R^2{\to}\R$ be a generalization gap. 
An algorithmic-dependent generalization bound is defined such that if for any distribution $\Dcal$ on $\Xcal\times\Ycal$, there exists a function $\PhiA: (0,1]{\to} \R$, such that for any  $\delta\in(0, 1]$ we have
\begin{align}
    \PP_{\Scal\sim\Dcal^m}\Big[ \phi(\RiskLoss_{\Dcal}(h_{\Scal}), \RiskLoss_{\Scal}(h_{\Scal})) \le \PhiA(\delta) \Big] \ge 1{-}\delta,\label{eq:algo}
\end{align}
where $h_{\Scal}\in\Hcal$ is the hypothesis learned from an algorithm with $\Scal\sim\Dcal^m$.
\end{restatable}

For example, when $\phi(\RiskLoss_{\Dcal}(h_{\Scal}), \RiskLoss_{\Scal}(h_{\Scal})) = \RiskLoss_{\Dcal}(h_\Scal) - \RiskLoss_{\Scal}(h_\Scal)$, the upper bound $\PhiA(\delta)=2\beta + (4m\beta+1)\sqrt{\frac{\ln1/\delta}{2m}}$ where $\beta$ is the uniform stability parameter~\citep[see,][]{bousquet2002stability} and a bounded loss $\loss: \Hcal{\times}(\Xcal{\times}\Ycal){\to} [0, 1]$.
Similarly to the uniform-convergence-based bounds, the upper bound $\PhiA(\delta)$ is a constant \wrt the hypothesis $h_{\Scal}$ and the learning sample $\Scal$.

\subsection{Obtaining Uniform-convergence Bounds}
\label{sec:obtaining-uc}

Since the parametric function $\comp$ in \Cref{theorem:disintegrated-comp} depends on the learning sample $\Scal$ and the hypothesis $h$, we can obtain from a specific $\comp$ a uniform-convergence-based bound (\Cref{eq:uc}).
Indeed, from \Cref{theorem:disintegrated-comp}, we obtain the following uniform-convergence-based bound.

\begin{restatable}{corollary}{corollarydisuc} \label{corollary:dis-uc}
Let $\loss: \Hcal{\times}(\Xcal{\times}\Ycal){\to}\R$ be a loss function, $\phi\!:\! \R^2{\to}\R$ be the generalization gap and assume that there exists a function $\PhiUC: (0,1]\to\R$ fulfilling \Cref{def:uc}.
Applying \Cref{theorem:disintegrated-comp} with the parametric function $\comp$ defined by
\begin{align*}
    \forall (h,\Scal)\in\Hcal{\times}(\Xcal{\times}\Ycal)^m,\quad \comp(h, \Scal) = -\phi(\RiskLoss_{\Dcal}(h),\RiskLoss_{\Scal}(h))  -\PhiUC(\tfrac{\delta}{2}) -\ln\P(h)
\end{align*}
gives the following bound 
\begin{align}
    \PP_{\Scal\sim\Dcal^m, h'\sim\P}&\Bigg[ \sup_{f\in\Hcal}\phi(\RiskLoss_{\Dcal}(f),\RiskLoss_{\Scal}(f)) \le \underbrace{\PhiUC(\tfrac{\delta}{2}) + \ln\!\left[\frac{16}{\delta^2} \EE_{\Vcal\!{\sim}\Dcal^m}\EE_{g{\sim}\P} e^{\phi(\RiskLoss_{\Dcal}(g),\RiskLoss_{\Vcal}(g))-\phi(\RiskLoss_{\Dcal}(h'),\RiskLoss_{\Scal}(h'))}\right]}_{\defeq\ \PhiUC'(\delta)} \Bigg] \ge 1{-}\delta.
\end{align}
\end{restatable}
\begin{proof}
Given the definition of $\AQ$ (with the parametric function $\comp$ defined above), we deduce from \Cref{theorem:disintegrated-comp} that 
\begin{align*}
    \PP_{\Scal\sim\Dcal^m,\ h'\sim\P, h\sim\AQ}\Bigg[ \phi(\RiskLoss_{\Dcal}(h), \RiskLoss_{\Scal}(h)) \le &\underbrace{-\phi(\RiskLoss_{\Dcal}(h'),\RiskLoss_{\Scal}(h'))  -\PhiUC(\tfrac{\delta}{2}) -\ln\P(h')}_{\comp(h', \Scal)}\\
    &+ \underbrace{\phi(\RiskLoss_{\Dcal}(h),\RiskLoss_{\Scal}(h)) + \PhiUC(\tfrac{\delta}{2}) + \ln\P(h)}_{-\comp(h, \Scal)}\\
    &+ \ln\frac{\P(h')}{\P(h)} + \ln\!\left[\frac{16}{\delta^2} \EE_{\Vcal\!{\sim}\Dcal^m}\EE_{g{\sim}\P} e^{\phi(\RiskLoss_{\Dcal}(g),\RiskLoss_{\Vcal}(g))}\right] \Bigg] \ge 1{-}\frac{\delta}{2}.
\end{align*}
Moreover, thanks to \Cref{def:uc}, with probability at least $1{-}\frac{\delta}{2}$ over the random choice of $\Scal$, we have $-\PhiUC(\tfrac{\delta}{2}) \le - \sup_{f\in\Hcal}\phi(\RiskLoss_{\Dcal}(f),\RiskLoss_{\Scal}(f))$.
Hence, with the union bound, we have that 
\begin{align*}
    \PP_{\Scal\sim\Dcal^m,\ h'\sim\P, h\sim\AQ}\Bigg[ \phi(\RiskLoss_{\Dcal}(h), \RiskLoss_{\Scal}(h)) \le &-\phi(\RiskLoss_{\Dcal}(h'),\RiskLoss_{\Scal}(h'))  -\sup_{f\in\Hcal}\phi(\RiskLoss_{\Dcal}(f),\RiskLoss_{\Scal}(f)) -\ln\P(h')\\
    &+ \phi(\RiskLoss_{\Dcal}(h),\RiskLoss_{\Scal}(h)) + \PhiUC(\tfrac{\delta}{2}) + \ln\P(h)\\
    &+ \ln\frac{\P(h')}{\P(h)} + \ln\!\left[\frac{16}{\delta^2} \EE_{\Vcal\!{\sim}\Dcal^m}\EE_{g{\sim}\P} e^{\phi(\RiskLoss_{\Dcal}(g),\RiskLoss_{\Vcal}(g))}\right] \Bigg] \ge 1{-}\delta.
\end{align*}
Therefore, by rearranging the terms, we obtain the desired result.
\end{proof}

\Cref{corollary:dis-uc} underlines the fact that our framework is general enough to allow us to obtain a uniform-convergence-based bound.
Indeed, if we are able to find (with high probability) an upper bound of the worst-case generalization gap $\sup_{f\in\Hcal}\phi(\RiskLoss_{\Dcal}(f),\RiskLoss_{\Scal}(f))$ denoted by $\PhiUC(\delta)$, then our framework allows us to obtain a bound depending on $\PhiUC(\delta)$.
For instance, when we consider the bound $\PhiUC(\delta)=\rad(\Hcal) {+} \sqrt{\frac{1}{2m}\ln\frac{1}{\delta}}$ depending on the Rademacher complexity $\rad(\Hcal)$, we are able to obtain a bound depending on $\PhiUC(\delta)$ thanks to our framework; it is shown in the following corollary.

\begin{corollary}\label{corollary:dis-uc-rad}
Let $\loss: \Hcal{\times}(\Xcal{\times}\Ycal){\to}[0,1]$ be a loss function.
By applying \Cref{theorem:disintegrated-comp} with the parametric function $\comp$ defined by
\begin{align*}
    \forall (h,\Scal)\in\Hcal{\times}(\Xcal{\times}\Ycal)^m,\quad \comp(h, \Scal) = -\sqrt{m}[\RiskLoss_{\Dcal}(h){-}\RiskLoss_{\Scal}(h)]  -\sqrt{m}\LB\rad(\Hcal) {+} \sqrt{\tfrac{1}{2m}\ln\tfrac{2}{\delta}}\RB -\ln\P(h),
\end{align*}
we can deduce the following bound 
\begin{align*}
\PP_{\Scal\sim\Dcal^m}&\Bigg[ \sup_{f\in\Hcal}\RiskLoss_{\Dcal}(f) - \RiskLoss_{\Scal}(f) \le \rad(\Hcal) {+} \sqrt{\frac{1}{2m}\ln\frac{4}{\delta}} + \frac{\ln\frac{128}{\delta^3}{+}2}{\sqrt{m}}\Bigg] \ge 1{-}\delta.
\end{align*}    
\end{corollary}
\begin{proof}
We apply \Cref{corollary:dis-uc} with the generalization gap $\phi(\RiskLoss_{\Dcal}(h),\RiskLoss_{\Scal}(h))=\sqrt{m}[\RiskLoss_{\Dcal}(h){-}\RiskLoss_{\Scal}(h)]$ and with $\PhiUC(\delta)=\rad(\Hcal) {+} \sqrt{\frac{1}{2m}\ln\frac{1}{\delta}}$ (see Theorem 3.3 of \citet{mohri2012foundations}). 
To obtain with probability at least $1-\delta/2$ over $\Scal\sim\Dcal^m$ and $h'\sim\P$ we have
\begin{align*}
\sup_{f\in\Hcal}\RiskLoss_{\Dcal}(f) - \RiskLoss_{\Scal}(f) \le \PhiUC(\tfrac{\delta}{4}) + \frac{1}{\sqrt{m}}\ln\!\left[\frac{64}{\delta^2} \EE_{\Vcal\!{\sim}\Dcal^m}\EE_{g{\sim}\P} e^{\sqrt{m}[\RiskLoss_{\Dcal}(g) - \RiskLoss_{\Vcal}(g)] - \sqrt{m}[\RiskLoss_{\Dcal}(h') - \RiskLoss_{\Scal}(h')]}\right].
\end{align*}
Moreover, thanks to Markov's inequality and the union bound, we obtain with probability at least $1-\delta$ over $\Scal\sim\Dcal^m$ we have
\begin{align*}
\sup_{f\in\Hcal}\RiskLoss_{\Dcal}(f) - \RiskLoss_{\Scal}(f) \le \PhiUC(\tfrac{\delta}{4}) + \frac{1}{\sqrt{m}}\ln\!\left[\frac{128}{\delta^3} \EE_{\Scal{\sim}\Dcal^m}\EE_{h'{\sim}\P}\EE_{\Vcal{\sim}\Dcal^m}\EE_{g{\sim}\P} e^{\sqrt{m}[\RiskLoss_{\Dcal}(g) - \RiskLoss_{\Vcal}(g)] - \sqrt{m}[\RiskLoss_{\Dcal}(h') - \RiskLoss_{\Scal}(h')]}\right].
\end{align*}
Then, we upper-bound the term $\EE_{\Scal\!{\sim}\Dcal^m}\EE_{h'{\sim}\P}\EE_{\Vcal\!{\sim}\Dcal^m}\EE_{g{\sim}\P} e^{[\RiskLoss_{\Dcal}(g) - \RiskLoss_{\Vcal}(g)] - [\RiskLoss_{\Dcal}(h') - \RiskLoss_{\Scal}(h')]}$.
To do so, we first use Fubini's theorem to have
\begin{align*}
&\EE_{\Scal{\sim}\Dcal^m}\EE_{h'{\sim}\P}\EE_{\Vcal{\sim}\Dcal^m}\EE_{g{\sim}\P} e^{\sqrt{m}[\RiskLoss_{\Dcal}(g) - \RiskLoss_{\Vcal}(g)] - \sqrt{m}[\RiskLoss_{\Dcal}(h') - \RiskLoss_{\Scal}(h')]}\\
&= \EE_{h'{\sim}\P}\EE_{g{\sim}\P}\EE_{\Scal{\sim}\Dcal^m}\EE_{\Vcal{\sim}\Dcal^m} e^{\sqrt{m}[\RiskLoss_{\Dcal}(g) - \RiskLoss_{\Vcal}(g)] - \sqrt{m}[\RiskLoss_{\Dcal}(h') - \RiskLoss_{\Scal}(h')]}.
\end{align*}
Moreover, we upper-bound the term $\EE_{\Scal{\sim}\Dcal^m}\EE_{\Vcal{\sim}\Dcal^m} e^{\sqrt{m}[\RiskLoss_{\Dcal}(g) - \RiskLoss_{\Vcal}(g)] - \sqrt{m}[\RiskLoss_{\Dcal}(h') - \RiskLoss_{\Scal}(h')]}$ thanks to Hoeffding's lemma, to obtain
\begin{align*}
&\EE_{\Vcal\!{\sim}\Dcal^m}\EE_{\Scal\!{\sim}\Dcal^m} \exp\LP\sqrt{m}[\RiskLoss_{\Dcal}(g) - \RiskLoss_{\Vcal}(g)] - \sqrt{m}[\RiskLoss_{\Dcal}(h') - \RiskLoss_{\Scal}(h')]\RP\\
&= \prod_{i=1}^{m}\LB \EE_{(\xbf_i',y_i')\sim\Dcal}\EE_{(\xbf_i,y_i)\sim\Dcal} \exp\LP\frac{1}{\sqrt{m}}\!\LB\EE_{(\xbf,z)\sim\Dcal}\loss(g, (\xbf,y)) - \loss(g, (\xbf'_i,y'_i)) - \EE_{(\xbf,z)\sim\Dcal}\loss(h', (\xbf,y)) + \loss(h', (\xbf_i,y_i))\RB\RP \RB\\
&\le \prod_{i=1}^{m}\LB \exp\LP\frac{2}{m}\RP\RB\\
&= \exp\LP2\RP.
\end{align*}
Hence, by rearranging the terms, we obtain the stated result.
\end{proof}

The bound that we can obtain in \Cref{corollary:dis-uc-rad} is greater than the bound of \citet{mohri2012foundations}'s Theorem 3.3.
This is normal since we use the bound in the parametric function $\comp$.
However, the higher the number of examples $m$, the closer our bound will be to the original bound of \citet{mohri2012foundations}.
Obtaining new uniform-convergence bound (without relying on previously known bounds) by setting a specific parametric function $\comp$ is highly non-trivial and is thus an exciting line of research that can be explored in the future.

\subsection{Obtaining Algorithmic-dependent Bounds}
\label{sec:obtaining-algo}

Similarly, we can obtain an algorithmic-dependent generalization bound (\Cref{def:algo}) by using the same technique as in \Cref{corollary:dis-uc}.
Indeed, we can obtain the following result.
\begin{restatable}{corollary}{corollarydisalgo} \label{corollary:dis-algo}
Let $\loss: \Hcal{\times}(\Xcal{\times}\Ycal){\to}\R$ be a loss function, $\phi\!:\! \R^2{\to}\R$ be the generalization gap and assume that there exists a function $\PhiA: (0,1]\to\R$ fulfilling \Cref{def:algo}.
Applying \Cref{theorem:disintegrated-comp} with the parametric function $\comp$ defined by
\begin{align*}
    \forall (h,\Scal)\in\Hcal{\times}(\Xcal{\times}\Ycal)^m,\quad \comp(h, \Scal) = -\phi(\RiskLoss_{\Dcal}(h),\RiskLoss_{\Scal}(h))  -\PhiA(\tfrac{\delta}{2}) -\ln\P(h)
\end{align*}
gives the following bound 
\begin{align}
    \PP_{\Scal\sim\Dcal^m, h'\sim\P}&\Bigg[ \phi(\RiskLoss_{\Dcal}(h_{\Scal}),\RiskLoss_{\Scal}(h_{\Scal})) \le \underbrace{\PhiA(\tfrac{\delta}{2}) + \ln\!\left[\frac{16}{\delta^2} \EE_{\Vcal\!{\sim}\Dcal^m}\EE_{g{\sim}\P} e^{\phi(\RiskLoss_{\Dcal}(g),\RiskLoss_{\Vcal}(g))-\phi(\RiskLoss_{\Dcal}(h'),\RiskLoss_{\Scal}(h'))}\right]}_{\defeq\ \PhiA'(\delta)} \Bigg] \ge 1{-}\delta.\label{eq:bound-dis-algo}
\end{align}
\end{restatable}
\begin{proof}
The proof is similar to the one of \Cref{corollary:dis-uc}.
Given the definition of $\AQ$ (with the parametric function $\comp$ defined above), we deduce from \Cref{theorem:disintegrated-comp} that 
\begin{align*}
    \PP_{\Scal\sim\Dcal^m,\ h'\sim\P, h\sim\AQ}\Bigg[ \phi(\RiskLoss_{\Dcal}(h), \RiskLoss_{\Scal}(h)) \le &\underbrace{-\phi(\RiskLoss_{\Dcal}(h'),\RiskLoss_{\Scal}(h'))  -\PhiA(\tfrac{\delta}{2}) -\ln\P(h')}_{\comp(h', \Scal)}\\
    &+ \underbrace{\phi(\RiskLoss_{\Dcal}(h),\RiskLoss_{\Scal}(h)) + \PhiA(\tfrac{\delta}{2}) + \ln\P(h)}_{-\comp(h, \Scal)}\\
    &+ \ln\frac{\P(h')}{\P(h)} + \ln\!\left[\frac{16}{\delta^2} \EE_{\Vcal\!{\sim}\Dcal^m}\EE_{g{\sim}\P} e^{\phi(\RiskLoss_{\Dcal}(g),\RiskLoss_{\Vcal}(g))}\right] \Bigg] \ge 1{-}\frac{\delta}{2}.
\end{align*}
Moreover, thanks to \Cref{def:algo}, with probability at least $1{-}\frac{\delta}{2}$ over the random choice of $\Scal$, we have $-\PhiA(\tfrac{\delta}{2}) \le -\phi(\RiskLoss_{\Dcal}(h_{\Scal}),\RiskLoss_{\Scal}(h_{\Scal}))$.
Hence, with the union bound, we have that 
\begin{align*}
    \PP_{\Scal\sim\Dcal^m,\ h'\sim\P, h\sim\AQ}\Bigg[ \phi(\RiskLoss_{\Dcal}(h), \RiskLoss_{\Scal}(h)) \le &-\phi(\RiskLoss_{\Dcal}(h'),\RiskLoss_{\Scal}(h'))  -\phi(\RiskLoss_{\Dcal}(h_{\Scal}),\RiskLoss_{\Scal}(h_{\Scal})) -\ln\P(h')\\
    &+ \phi(\RiskLoss_{\Dcal}(h),\RiskLoss_{\Scal}(h)) + \PhiA(\tfrac{\delta}{2}) + \ln\P(h)\\
    &+ \ln\frac{\P(h')}{\P(h)} + \ln\!\left[\frac{16}{\delta^2} \EE_{\Vcal\!{\sim}\Dcal^m}\EE_{g{\sim}\P} e^{\phi(\RiskLoss_{\Dcal}(g),\RiskLoss_{\Vcal}(g))}\right] \Bigg] \ge 1{-}\delta.
\end{align*}
Therefore, by rearranging the terms, we obtain the desired result.
\end{proof}

Hence, our framework is also general enough to retrieve algorithmic-dependent bounds.
More precisely, the generalization gap $\phi(\RiskLoss_{\Dcal}(h_{\Scal}),\RiskLoss_{\Scal}(h_{\Scal}))$ associated with the hypothesis $h_{\Scal}$ is upper-bounded by a constant $\PhiA'(\delta)$.
As for \Cref{corollary:dis-uc,corollary:dis-uc-rad}, the drawback of \Cref{corollary:dis-algo} is that we have to rely on a previously known bound to obtain our result.
Therefore, further investigations must be done to derive entirely new algorithmic-depend bounds by setting a specific parametric function $\comp$.

\section{ADDITIONAL INFORMATION ON THE EXPERIMENTS}
\label{sec:additional-experiments}

In this section, we first provide more experiments about \Cref{sec:experiments-reg-risk} by varying $\alpha$.
\Cref{sec:additional-experiments-neural-comp} presents how $\neural$ is obtained and more experiments on this parametric function (along with $\distneural$).

\subsection[About Computing the Bounds]{About Computing the Bounds (with $\klmax$)}

The evaluated bounds that we consider have all the same structure: with high probability, we have $\kl(q\|p) \le \tau$, where $q$ is the empirical risk, $p$ is the true risk, and $\tau$ is the bound.
As shown, \eg, in \Cref{eq:disintegrated-comp-seeger-risk}, we can evaluate the bound on the true risk $p$ by computing 
\begin{align*}
\klmax[q | \tau]=\max\Big\{ p \in (0,1) \;\Big|\; \kl(q\|p) \le \tau\Big\}.
\end{align*}
We use the bisection method to solve this optimization problem, as proposed by \citet{reeb2018learning}.
This method consists of refining the interval $[p_{\text{min}}, p_{\text{max}}]$ in which $p$ belongs.
To do so, we first initialize $p_{\text{min}}=q$ and $p_{\text{max}}=1$.
Then, for each iteration, we first set $p_{\text{tmp}}=\frac{1}{2}(p_{\text{max}}-p_{\text{min}})$, and then we change the values of $p_{\text{min}}$ and $p_{\text{max}}$ depending on the value of the temporary parameter $p_{\text{tmp}}$.
Indeed, if $\kl(q\|p_{\text{tmp}})>\tau$ (\resp, $\kl(q\|p_{\text{tmp}})<\tau$), we update $p_{\text{max}}=p_{\text{tmp}}$ (\resp, $p_{\text{min}}=p_{\text{tmp}}$).
Moreover, if we have $\kl(q\|p_{\text{tmp}})=\tau$, a small interval (with $p_{\text{max}}-p_{\text{min}}<\epsilon$), or if we attain the maximum number of iterations, we return $p_{\text{tmp}}$ as $\klmax[q | \tau]$.
In the experiments, we set $\epsilon=10^{-9}$ and the maximum number of iterations to $1000$.

\subsection[About Section 4.3]{About \Cref{sec:experiments-reg-risk}}
\label{sec:additional-experiments-reg-risk}

In the experiments introduced in \Cref{sec:experiments-reg-risk}, we fix $\alpha=m$. 
We propose additional experiments in \Cref{fig:reg-risk-alpha,fig:reg-risk-alpha-prior} where $\alpha$ varies between $\sqrt{m}$ and $m$.
As we can remark in \Cref{fig:reg-risk-alpha,fig:reg-risk-alpha-prior}, the concentration parameter $\alpha$ plays an important role: the higher $\alpha$, the lower the test risks $\RiskLoss_{\Tcal}(h)$.
Moreover, the bounds are tighter than those in \Cref{fig:reg-risk}; however, the test risks remain high, so the bounds are not sufficiently concentrated.
In this setting, there is a trade-off between concentrating the distribution (by increasing $\alpha$) and having a tight bound.

\subsection[About Section 4.4]{About \Cref{sec:experiments-neural-comp}}
\label{sec:additional-experiments-neural-comp}

In this section, we first introduce in \Cref{sec:additional-experiments-neural-comp-training} the setting to learn the parametric functions $\neural$ and $\distneural$ with neural networks, and we show additional experiments in \Cref{sec:additional-experiments-neural-comp-experiments}.

\subsubsection{Training the Neural Parametric Functions}
\label{sec:additional-experiments-neural-comp-training}

\textbf{$\neural$'s dataset.} In order to learn the neural networks associated with $\neural$, we have to train first neural networks (to have the weights as input) and save their corresponding generalization gap (that is further used as a label).
To do so, for MNIST and FashionMNIST, we train models with a size of the validation set that varies in order to obtain models with diverse generalization gaps.
Starting from the original training set of MNIST or FashionMNIST, we split the dataset into a training set and a validation set (to compute the gaps); we denote by $m_{\text{val}}$ the size of the validation set and $m_{\text{train}}$ the size of the training set.
After fixing the split, we launch training and save the model with its corresponding gap after each epoch. 
We launch $1000$ trainings with the split ratio $\frac{m_{\text{val}}}{m_{\text{val}}+m_{\text{train}}} \in \{ 0.99, 0.97, 0.95, 0.93 \}$, $120$ trainings with $\frac{m_{\text{val}}}{m_{\text{val}}+m_{\text{train}}}=0.90$ and $110$ trainings with $\frac{m_{\text{val}}}{m_{\text{val}}+m_{\text{train}}}\in\{ 0.80, 0.70, 0.60, 0.50, 0.40, 0.30, 0.20, 0.10\}$.
In each training, we learn the model with the same architecture as in \Cref{sec:experiments-setting}, and we optimize in the same way as for the sampling from $\AQ$, except that we run SGD instead of SGLD (\ie, we remove the Gaussian noise) for at least 8000 iterations (we finish the epoch after reaching the number of iterations).
We show in \Cref{fig:hist} the histogram of the obtained generalization gaps.
Note that our method of creating the dataset is similar to \citet{lee2020neural}, except that the size of our validation set varies more, and we save the parameters directly (instead of the predictions).

\textbf{$\neural$'s model.} 
As summarized in \Cref{sec:experiments-neural-comp}, the parametric function $\neural$ is a neural network that is learned with the dataset created previously.
This model is a feed-forward neural network with $3$ hidden layers of width $1024$.
The input is the weights and biases $\wbf$ of the network $h$, whereas the output is a scalar representing $\neural(h,\Scal)$, \ie, the value of the parametric function learned from the neural network.
After the input, we normalize the parameters $\wbf$ with its $\ell_2$ norm, and we use a batch normalization layer~\citep{ioffe2015batch} (with a momentum of $0.1$ and $\epsilon=0.0$). 
Moreover, the activation functions are leaky ReLU, and the output is squared to obtain a positive output.
The weights are initialized with the Xavier Glorot uniform initializer~\citep{glorot2010understanding}.
The biases are initialized with a uniform distribution between $-1/\sqrt{1024}$ and $+1/\sqrt{1024}$ for all biases, except for the first layer, they are initialized uniformly in $[-1/\sqrt{d}, +1/\sqrt{d}]$, where $d$ is the number of parameters of the models in the dataset.
The model is learned from Adam optimizer~\citep{kingma2015adam} for 100 epochs by minimizing the mean absolute error.
The model is selected by early stopping: the $\neural$'s dataset is split between a training set (of size $m_{\text{train}}$) and a validation set (of size $m_{\text{val}}$).
Moreover, in order to handle the fact that the $\neural$'s dataset is unbalanced, we rebalance it by {\it (i)} putting the gaps into $50$ bins, {\it (ii)} merging neighboring bins if they represent less than 1\% of the dataset and {\it (iii)} sampling a bin with a probability proportional to the inverse of the number of examples in the bin (the examples in the bin are sampled uniformly).
The merging procedure is done as follows: we merge the bin with its neighbor (that contains higher gaps) when it has less than 1\% of the dataset; we perform this until stabilization.
We learn different networks (that give several parametric functions $\neural$), with a batch size of $64$, $128$, or $256$; Adam's learning rate is either $0.001$ or $0.0001$ (the other parameters remain the default ones); and the ratio $\frac{m_{\text{val}}}{m_{\text{train}}+m_{\text{val}}}$ is either $0.1$, $0.3$ or $0.5$.
Note that the work of \citet{lee2020neural} differs from ours by the fact that {\it (i)} we have a simpler architecture, and {\it (ii)} we take the parameters as input in order to have a differentiable parametric function (for sampling from $\AQ$).

\subsubsection{Additional Experiments}
\label{sec:additional-experiments-neural-comp-experiments}

In \Cref{fig:neural-mnist,fig:neural-fashion,fig:dist-neural-mnist,fig:dist-neural-fashion}, we show the bar plots of the parametric functions $\neural$ and $\distneural$ learned with the different hyperparameters.
The figures highlight that hyperparameter tuning is extremely important.
Indeed, the performance of the parametric functions can change drastically between two sets of hyperparameters.
We believe that understanding the role of these hyperparameters is an exciting future work that might improve the performance of such parametric functions.

\newpage

\begin{figure*}
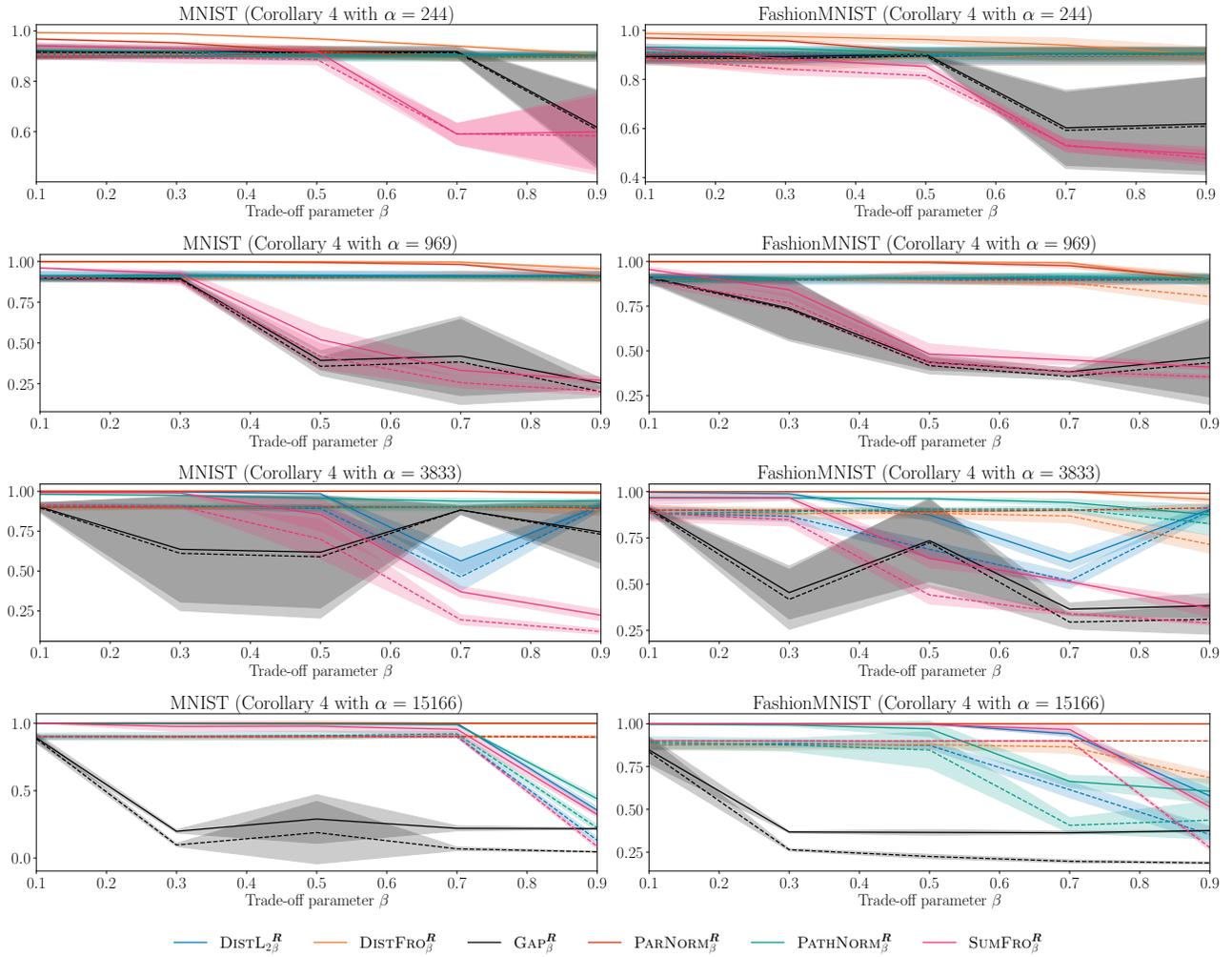

    \centering
    \includestandalone[width=1.0\linewidth]{figures/fig_5}
    \caption{
    Evolution of the bounds (the plain lines) and the test risks $\RiskLoss_{\Tcal}(h)$ (the dashed lines) \wrt the trade-off parameter $\beta$ for varying $\alpha$ and $\frac{m'}{m'+m}=0.0$.
    The lines correspond to the mean, while the bands are the standard deviations.
    }
    \label{fig:reg-risk-alpha}
\end{figure*}

\begin{figure*}
    \centering
    \includestandalone[width=1.0\linewidth]{figures/fig_5_prior}
    \caption{
    Evolution of the bounds (the plain lines) and the test risks $\RiskLoss_{\Tcal}(h)$ (the dashed lines) \wrt the trade-off parameter $\beta$ for varying $\alpha$ and $\frac{m'}{m'+m}=0.5$.
    The lines correspond to the mean, while the bands are the standard deviations.
    }
    \label{fig:reg-risk-alpha-prior}
\end{figure*}

\begin{figure}
\centering
\includegraphics[width=1.0\linewidth]{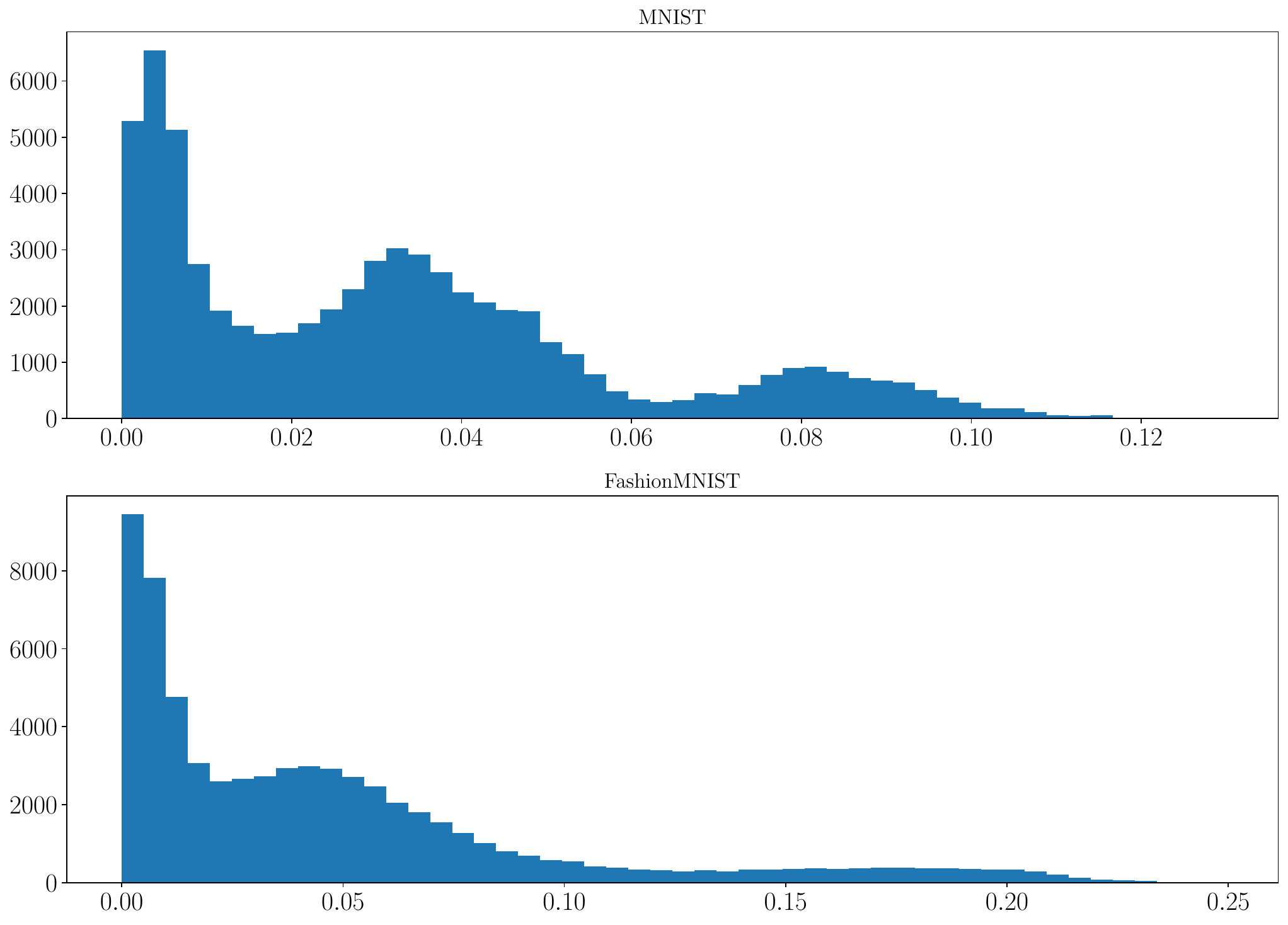}
\caption{
Histograms of the generalization gaps associated with the neural networks in the dataset to learn $\neural$. 
}
\label{fig:hist}
\end{figure}

\begin{figure}
\centering
\includegraphics[width=0.9\linewidth]{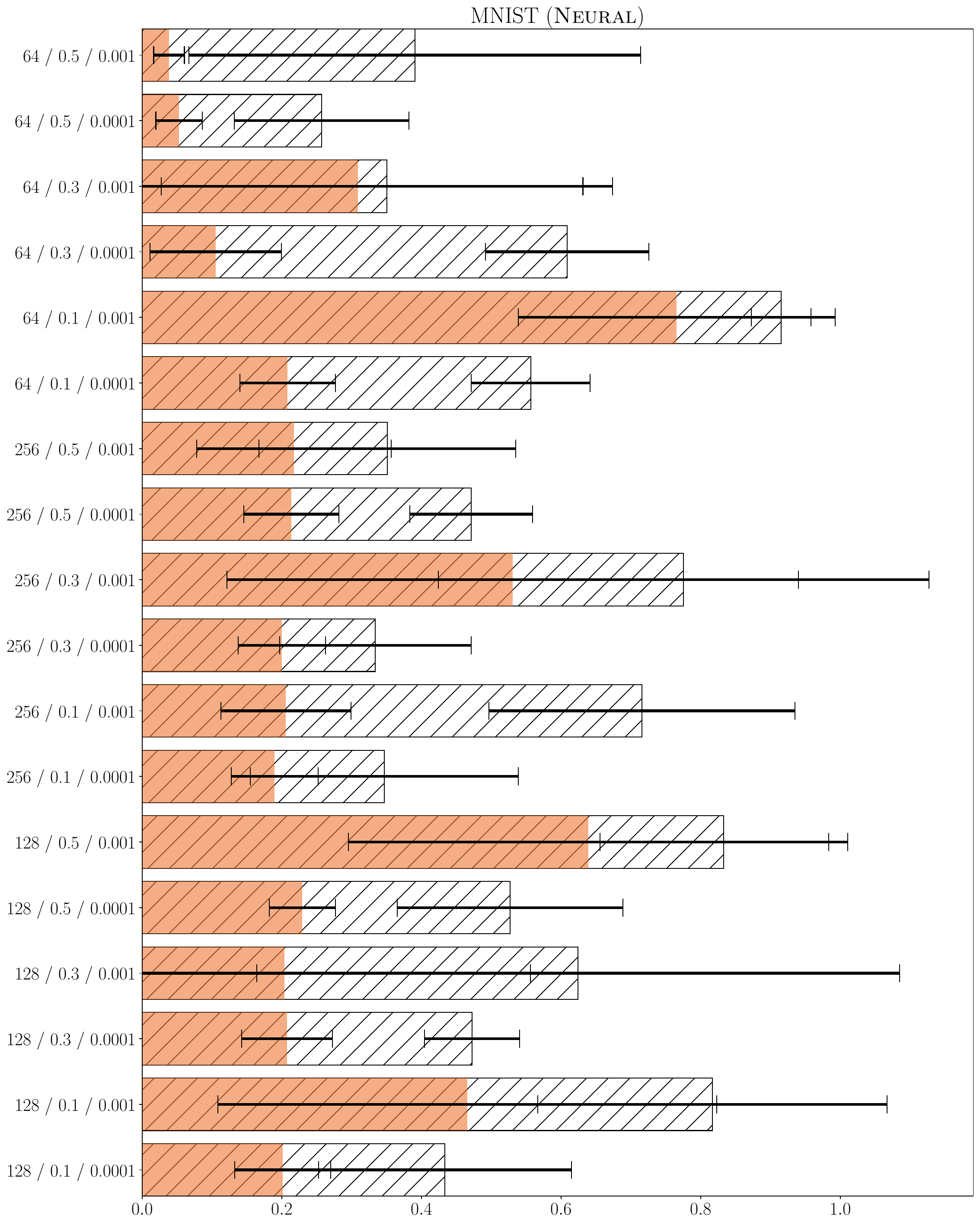}
\caption{
Bar plot of the bound value associated with \Cref{corollary:disintegrated-comp-unif} and MNIST for the parametric function $\neural$ learned with different hyperparameters.
On the y-axis, the bar labels ``A / B / C'' represent the three hyperparameters that vary: ``A'' is the batch size, ``B'' is the size of the validation set compared to the original dataset (of neural networks), and ``C'' is the learning rate of the Adam optimizer.
The mean bound values of the sampled hypotheses $h\sim\AQ$ are shown with the hatched bars, and the mean test risks $\RiskLoss_{\Tcal}(h)$ are plotted in the colored bars.
Moreover, the standard deviations are plotted in black.
}
\label{fig:neural-mnist}
\end{figure}

\begin{figure}
\centering
\includegraphics[width=0.9\linewidth]{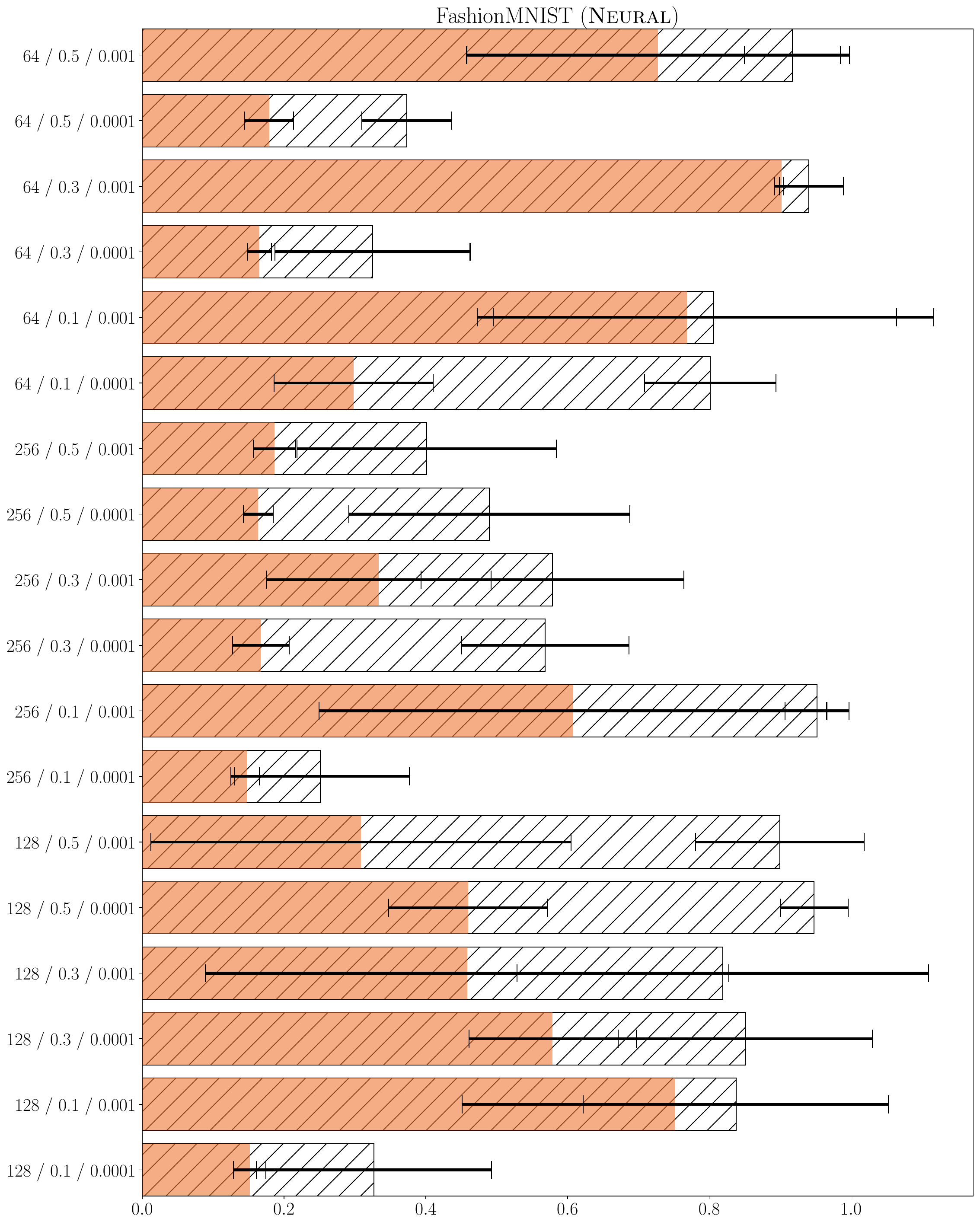}
\caption{\label{fig:neural-fashion}
Bar plot of the bound value associated with \Cref{corollary:disintegrated-comp-unif} and FashionMNIST for the parametric function $\neural$ learned with different hyperparameters.
On the y-axis, the bar labels ``A / B / C'' represent the three hyperparameters that vary: ``A'' is the batch size, ``B'' is the size of the validation set compared to the original dataset (of neural networks), and ``C'' is the learning rate of the Adam optimizer.
The mean bound values of the sampled hypotheses $h\sim\AQ$ are shown with the hatched bars, and the mean test risks $\RiskLoss_{\Tcal}(h)$ are plotted in the colored bars.
Moreover, the standard deviations are plotted in black.
}
\end{figure}

\begin{figure}
\centering
\includegraphics[width=0.9\linewidth]{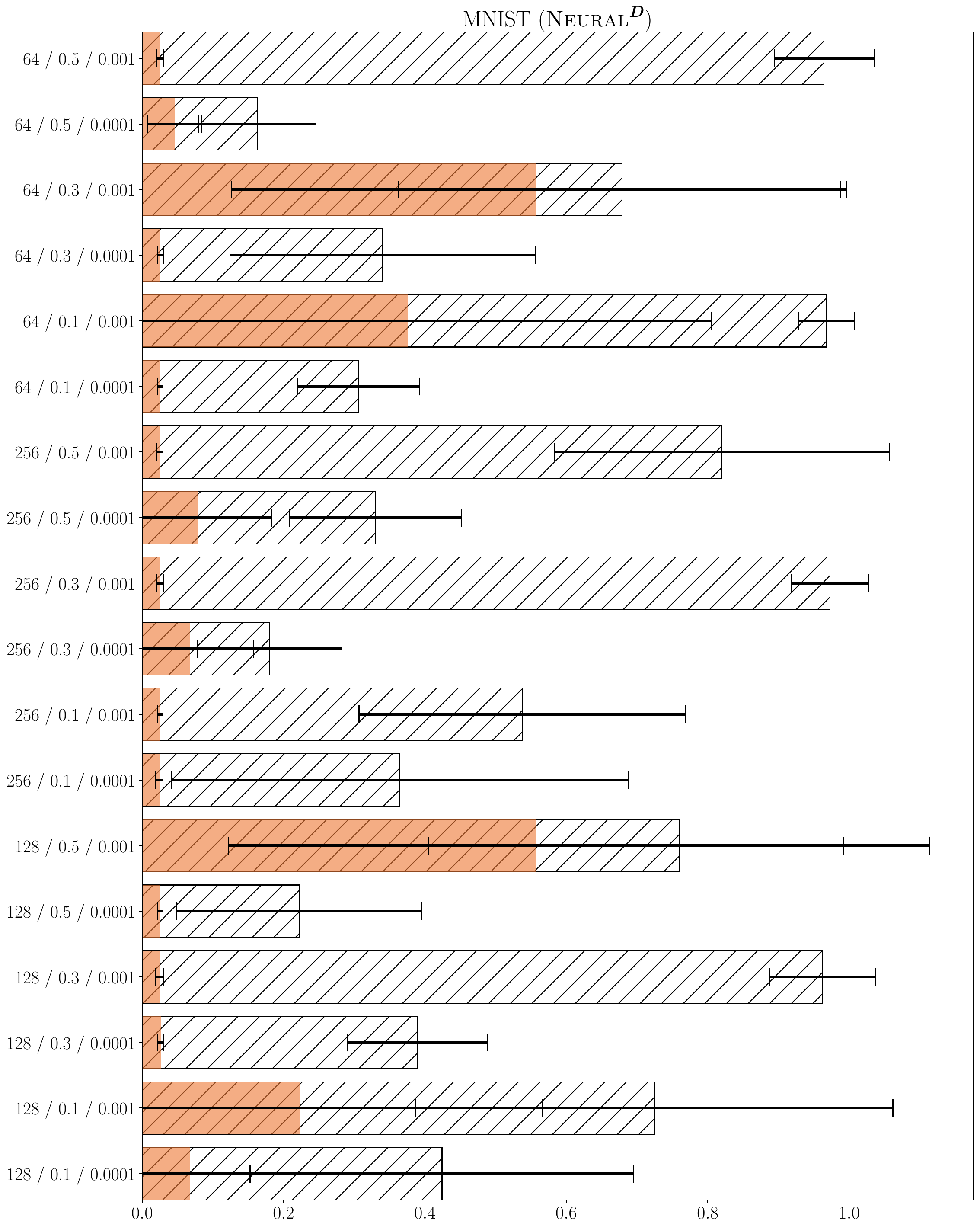}
\caption{
Bar plot of the bound value associated with \Cref{corollary:disintegrated-comp-unif} and MNIST for the parametric function $\distneural$ learned with different hyperparameters.
On the y-axis, the bar labels ``A / B / C'' represent the three hyperparameters that vary: ``A'' is the batch size, ``B'' is the size of the validation set compared to the original dataset (of neural networks), and ``C'' is the learning rate of the Adam optimizer.
The mean bound values of the sampled hypotheses $h\sim\AQ$ are shown with the hatched bars, and the mean test risks $\RiskLoss_{\Tcal}(h)$ are plotted in the colored bars.
Moreover, the standard deviations are plotted in black.
}
\label{fig:dist-neural-mnist}
\end{figure}

\begin{figure}
\includegraphics[width=0.9\linewidth]{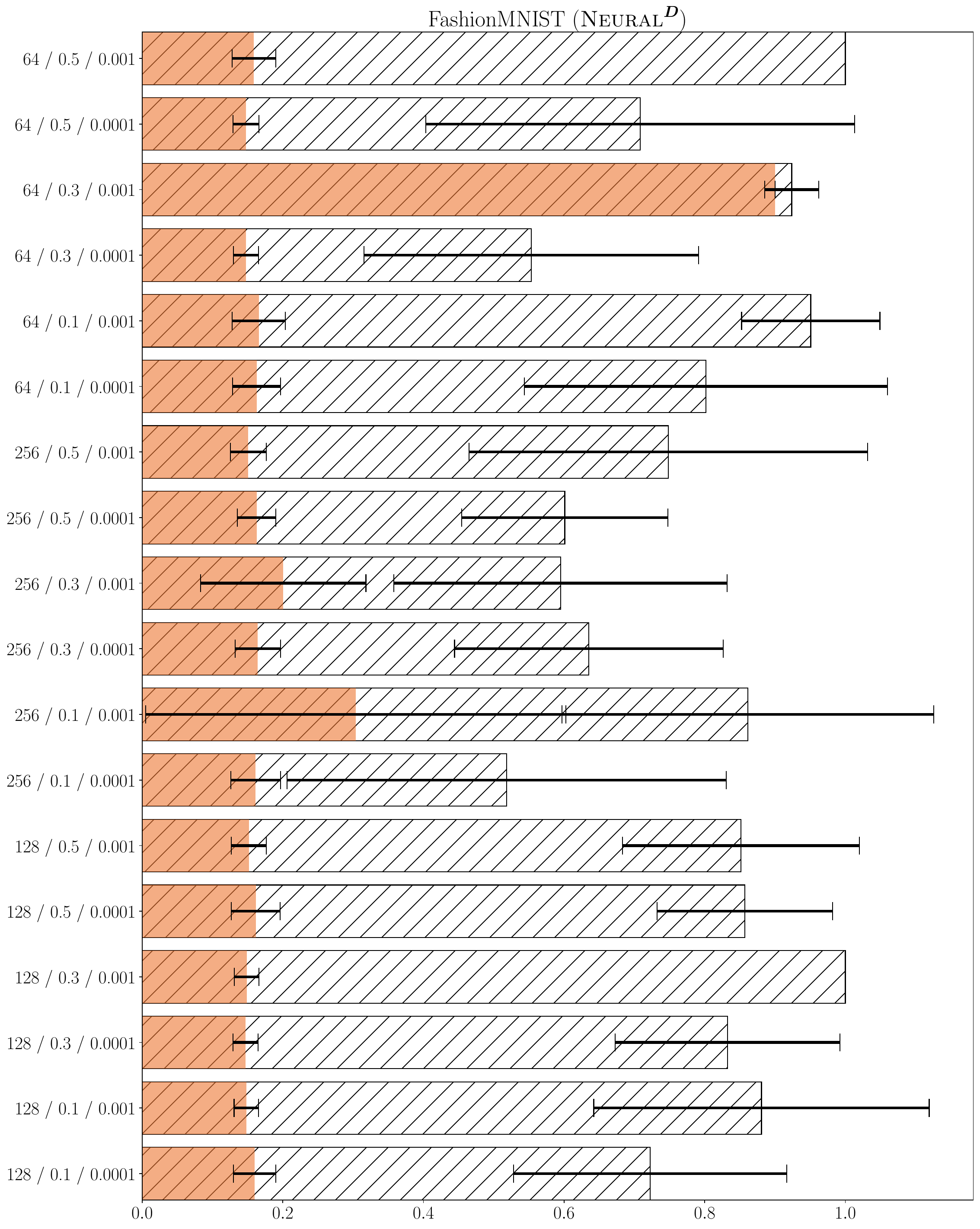}
\caption{
Bar plot of the bound value associated with \Cref{corollary:disintegrated-comp-unif} and FashionMNIST for the parametric function $\distneural$ learned with different hyperparameters.
On the y-axis, the bar labels ``A / B / C'' represent the three hyperparameters that vary: ``A'' is the batch size, ``B'' is the size of the validation set compared to the original dataset (of neural networks), and ``C'' is the learning rate of the Adam optimizer.
The mean bound values of the sampled hypotheses $h\sim\AQ$ are shown with the hatched bars, and the mean test risks $\RiskLoss_{\Tcal}(h)$ are plotted in the colored bars.
Moreover, the standard deviations are plotted in black.
}
\label{fig:dist-neural-fashion}
\end{figure}

\end{document}


\begin{tikzpicture}

    \node[] at (0.0,0.0) {\includegraphics[scale=1.0]{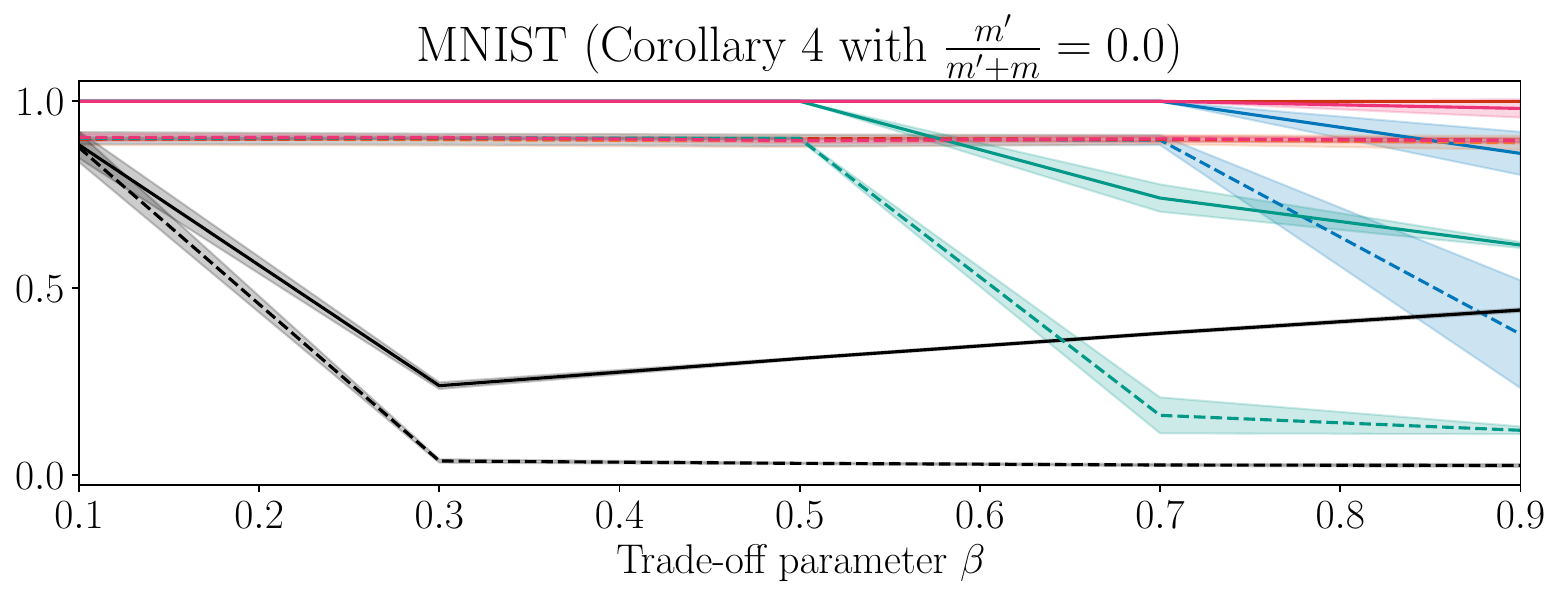}};
    \node[] at (26,0.0) {\includegraphics[scale=1.0]{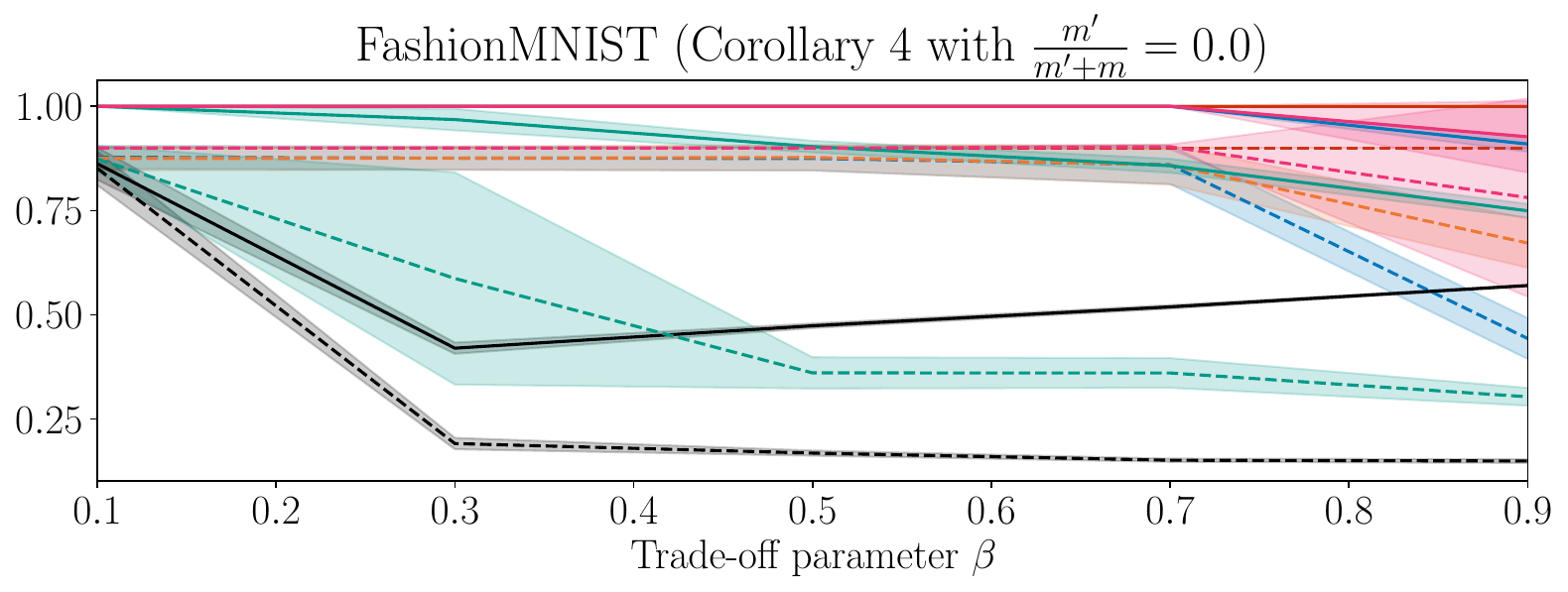}};
    \node[] at (13.0,-5.0) {\includegraphics[scale=1.0]{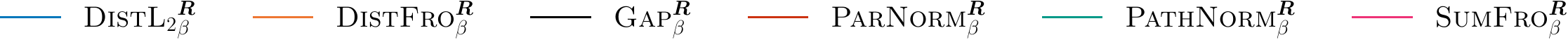}};
\end{tikzpicture}